%% file: main.tex
\documentclass[10pt]{article} %
\usepackage[preprint]{tmlr}

\usepackage{hyperref}
\usepackage{url}

\title{Maximum Mean Discrepancy on Exponential Windows for Online Change Detection}

\author{\name Florian Kalinke \email florian.kalinke@kit.edu\\
      \addr Karlsruhe Institute of Technology, Germany
      \AND
      \name Marco Heyden \email marco.heyden@kit.edu\\
      \addr Karlsruhe Institute of Technology, Germany
      \AND
      \name Georg Gntuni \email g.gntuni@gmail.com\\
      \addr Karlsruhe Institute of Technology, Germany
      \AND
      \name Edouard Fouché \email edouard.fouche@kit.edu\\
      \addr Karlsruhe Institute of Technology, Germany
      \AND
      \name Klemens Böhm \email klemens.boehm@kit.edu\\
      \addr Karlsruhe Institute of Technology, Germany}

\def\month{01}  %
\def\year{2025} %
\def\openreview{\url{https://openreview.net/forum?id=OGaTF9iOxi}} %

\input{preamble.tex}
\begin{document}

\newcommand{\mysize}{.95}

\newcounter{tmp} %

\maketitle

\begin{abstract}
	Detecting changes is of fundamental importance when analyzing data streams
	and has many applications, e.g., in predictive maintenance, fraud detection, or
	medicine. A principled approach to detect changes is to compare the
	distributions of observations within the stream to each other via hypothesis
	testing. Maximum mean discrepancy (MMD),
	a (semi-)metric on the space of probability distributions,
	provides powerful non-parametric two-sample tests on kernel-enriched domains. In particular, MMD is able to detect any
	disparity between distributions under mild conditions. However, classical MMD estimators suffer from a quadratic runtime complexity, which renders their direct
	use for change detection in data streams impractical. In this article, we propose a new change detection algorithm, called Maximum Mean Discrepancy on Exponential Windows
	(MMDEW), that combines the benefits of MMD with an efficient computation based on exponential windows. We prove that MMDEW enjoys polylogarithmic
	runtime and logarithmic memory complexity
	and show empirically that it outperforms the state of the art on benchmark data streams.
\end{abstract}

\section{Introduction}
Data streams are possibly infinite sequences of observations that arrive over time. They can have different sources: sensors in industrial settings, online transactions from financial institutions, click monitoring on websites, online feeds, etc.
Quickly detecting when a change takes place can yield useful insights, for example, about machine failure, malicious financial transactions, changes in customer preferences, and public opinions.

A \textit{change} occurs if the underlying distribution of the data stream changes at a certain point in time.
We call this moment \textit{change point}~\citep{gama2010knowledgediscovery}; it is sometimes also referred to as \textit{concept drift}.
A principled and widely-used approach to detect changes is to use two-sample tests. The null hypothesis of such tests is that the data before and after the potential change point follow the same distribution.
If the test rejects the hypothesis, one assumes that a change occurred.

One way to construct these tests is to use the kernel-based maximum mean discrepancy (MMD; \citealt{smola07hilbert,gretton2012kerneltwosample}),
which one can interpret as a (semi-)metric on the space of probability distributions.\footnote{A function is a semimetric if it is a metric but can be zero for distinct elements.}
In the statistics literature, MMD is also known as energy distance \citep{szekely04testing,szekely05new}; see \citet{sejdinovic13equivalence} for the equivalence.
MMD relies on the kernel mean embedding \citep[Ch.~4]{berlinet04reproducing}; it uses a kernel function to map a probability distribution to a reproducing kernel Hilbert space (RKHS; \citealt{aronszajn50theory}) and quantifies the discrepancy of the two distributions as their distance in the RKHS. MMD is a metric if the kernel mean embedding is injective; the kernel is then called characteristic \citep{fukumizu08kernel,sriperumbudur10hilbert}.
When using a characteristic kernel, the MMD two-sample test allows to distinguish any distributions given that their kernel mean embeddings exist, which is guaranteed under mild conditions.

Two-sample tests based on MMD are widely applicable, as there exist kernel
functions for a multitude of Euclidean and non-Euclidean domains, for example, strings
\citep{watkins99dynamic,cuturi05contextfree}, graphs
\citep{gartner03graph,borgwardt20graph}, or time series
\citep{cuturi11fast,kiraly19kernel}. Another benefit
of kernel-based two-sample tests is their high power. While, for Euclidean data,
it has been shown that the power of such tests generally decreases in the
high-dimensional setting \citep{ramdas2015powerhypothesistest}, recent results
\citep{cheng24kernel} establish that the power rather depends on
the intrinsic dimensionality of the data. The intrinsic dimensionality is
typically low in real-world settings so that the kernel-based two-sample tests
there do not suffer the curse of dimensionality.

Despite these benefits, a well-known bottleneck of MMD-based
approaches is their computational complexity. When
comparing the distributions of two sets of data of sizes $m$ and $n$,
respectively, the computation of MMD with classical estimators is in $\O\left(m^2 + n^2\right)$, with a
memory complexity in $\O\left(m+n\right)$. Naively computing MMD for each
possible change point on a data stream with $t=m+n$ observations has a complexity in
$\O\left(t^3\right)$ for each new observation. These properties render the direct application of MMD to change detection in data streams impractical.

In this paper, we introduce Maximum Mean Discrepancy on Exponential Windows
(MMDEW), a change detection algorithm for data streams that solves the above bottleneck.
Specifically, our \tb{contributions} include the following.

\begin{itemize}
	\item Our main contribution is MMDEW, a change detector based on an efficient online approximation of MMD. When considering the entire history of $t$ observations, the proposed method has a memory requirement of $\O\left(\log t\right)$ and a runtime complexity of $\O\left(\log^2 t\right)$ for each new observation. Otherwise, the algorithm has constant runtime and memory requirements.
	\item To achieve these complexities, we introduce a new data structure, which allows to approximate the quadratic time MMD in an online setting. We accomplish  the speedup by introducing windows that store summaries of the observations seen so far, and by storing a sample of logarithmic size of the observations per window.

	\item Our experiments on standard benchmark data sets show
	      that MMDEW performs better than state-of-the-art change detectors on four out of
	      the five tested data sets using the $F_1$-score.
	      For the more challenging setting of short detection delays, the proposed algorithm is better on three
	      out of six data sets.\footnote{Our code is available at \url{https://github.com/FlopsKa/mmdew-change-detector}.} %

\end{itemize}

\tb{Outline.} Section~\ref{sec:related-work} summarizes related work. Section~\ref{sec:definitions}
introduces the definitions and Section~\ref{sec:MMDEW} presents the proposed algorithm. We detail the experiments in Section~\ref{sec:experiments}. Section~\ref{sec:conclusions} concludes. We include illustrative proofs in the main text but defer technical proofs, additional details, and additional experiments to the appendices.

\section{Related work} \label{sec:related-work}

Change detection is an unsupervised task that has received and still is receiving a lot of interest. The earliest approaches, for example, \citet{shewhart25application,page54continuous}, originated from quality control and require strong parametric assumptions on the pre and post-change distributions. More recent work in the parametric regime weakens these assumptions by allowing post-change distributions from a parametric family with an unknown parameter \citep{lorden70glr,siegmund95changedetection} or by allowing any post-change distribution \citep{sparks00cusum,lorden05change,abbasi19cusum,xie23window}.

In our setting, both the pre and post-change distribution are assumed to be unknown, which is a challenging setting that can be tackled with non-parametric approaches. We detail the approaches most related to our proposed method in the following and refer to \citet{wang24sequential} for a recent more extensive survey on parametric and non-parametric change detection methods.

A principled approach for comparing distributions in a stream in a non-parametric fashion is to use a corresponding statistical test. ADWIN \citep{DBLP:conf/sdm/BifetG07} is a classic example but it is limited to univariate data and only detects changes in mean. ADWINK \citep{DBLP:journals/inffus/FaithfullDK19} alleviates the former by running one instance of ADWIN per feature and issues a change if a predefined number of the instances agree that a change occurred. Hence, the approach can only detect changes in the means of the marginal distributions and changes in higher moments or the covariance structure can not be detected.
Still, the authors find that such an ensemble of univariate change detectors often outperforms multivariate detectors.
WATCH \citep{DBLP:conf/bigdataconf/FaberCSBJ21} is a recent approach that uses a two-sample test based on the Wasserstein distance. However, the estimation of the Wasserstein distance requires density estimation, which is difficult for high-dimensional data \citep{scott1991feasibility}.
The method \citet{DBLP:conf/ida/DasuKLVY09} is conceptually similar to our method, as it also relies on two-sample tests and is non-parametric, but it also requires density estimation.

In contrast, the computation of MMD-based two-sample tests does not become more difficult on high-dimensional data,
which renders their usage for change detection on such data promising. We refer to \citet{DBLP:journals/ftml/MuandetFSS17} for a general overview of kernel mean embeddings and MMD.

There exist methods to compute MMD in the streaming setting, for example, linear time tests \citep{gretton2012kerneltwosample}, but their statistical power is low. \citet{DBLP:conf/nips/ZarembaGB13} introduce $B$-tests, which have higher power. However, both can not directly be used for change detection.
\citet{li2019scanbstatistics} enable the estimation of MMD on data streams for
change detection by introducing Scan $B$-statistics. \citet{wei22online} extend upon their work by considering multiple Scan $B$-statistics in parallel and introduce online kernel CUSUM. Another method enabling the computation of MMD on data streams is NEWMA \citep{keriven2020mmdchangedetection}, which is based on random
Fourier features \citep{DBLP:conf/nips/RahimiR07,sriperumbudurszabo15optimal}, a well-known kernel
approximation. NEWMA also allows detecting changes on streaming data. %
\citet{harchaoui07retrospective} apply kernel-based tests for offline change point detection on audio and
brain-computer-interface data.

A conceptually different approach to find changes is using classifiers.
D3 \citep{DBLP:conf/cikm/GozuacikBBC19} maintains two consecutive sliding windows
and trains a classifier to distinguish their elements. It reports a change if
the classifier performance, measured by AUC, drops below a threshold.
Another recent
algorithm is IBDD~\citep{DBLP:journals/kais/SouzaPCM21}, which scales well with
the number of features.

In our experiments in the main text, we compare MMDEW to ADWINK, WATCH, Scan $B$-Statistics,
NEWMA, D3, and IBDD as these allow change detection on multivariate streams (in $\R^d$). These algorithms differ w.r.t.\ their runtime complexity, their theoretical
properties, the data types that they can handle, and the types of changes that they can
detect. We summarize their main properties in Table~\ref{tab:comparison}.\footnote{$^a$We refer to their used implementation of the Wasserstein distance computation and the discussion therein~\citep[Ch.~6]{merigot11multiscale}.
	$^b$$m$ is the number of random Fourier features and $m\ll d$.
	$^c$The complexity results from the matrix inversion of the logistic regression model, which has cubic runtime cost.
	$^d$Size of the constructed $q\times p$ image.}
We consider the dimensionality $d$ as constant for the complexities where its influence is dominated by other terms
and for approaches not restricted to Euclidean domains. In Appendix~\ref{appendix:edd-mtd-comparison} and Appendix~\ref{appendix:univariate-comparison}, we collect additional experiments on synthetic data. In particular, we additionally compare the proposed method to online kernel CUSUM and to multivariate adaptations of the Cramer-von-Mises change point model (CvM CPM; \citealt{ross12controlcharts}) and non-parametric Focus \citep{romano23changedetection}.

\begin{table}
	\centering
	\caption{Comparison of change detectors.  Complexity --- runtime complexity per new
		observation, ARL / MTD --- type of known results, domain ---
		data types, $t$ ---
		total number of observations, $d$ --- dimensionality (for Euclidean spaces), $k$
		--- parameter, $W$ --- window length / block size, $N$ --- number of windows.}

	\begin{tabular}{lccc}
		\toprule
		Algorithm  & Complexity                & ARL / MTD   & Domain      \\ %
		\midrule
		ADWINK %
		           & $\O\left(dk\log W\right)$ & empirical   & $\R^d$      \\ %
		WATCH      & unknown$^a$               & empirical   & $\R^d$      \\ %
		Scan $B$   & $\O\left(NW^2\right)$     & analytical  & topological \\ %
		NEWMA      & $\O\left(md\right)$$^b$   & analytical  & $\R^d$      \\ %
		D3         & $\O\left(W^3\right)$$^c$  & none        & $\R^d$      \\
		IBDD       & $\O\left(pq\right)$$^d$   & none        & $\R^d$      \\ %
		\tb{MMDEW} & $\O\left(\log^2 t\right)$
		           & empirical                 & topological               \\
		\bottomrule
	\end{tabular}

	\label{tab:comparison}
\end{table}

\section{Definitions and background}
\label{sec:definitions}
This section defines our problem and recalls kernels, the mean embedding, maximum mean discrepancy, and two-sample testing.

\tb{Problem definition.}
Let $(\X, \tau_{\X})$ be a topological space, $\mathcal B(\tau_{\X})$ the Borel sigma-algebra induced by $\tau_{\X}$, and $\mathcal M_1^+(\X)$ the set of probability measures on $\X$ meant w.r.t.\ the measurable space $(\X, \mathcal B(\tau_\X))$.
We consider a data stream, that is, a possibly infinite sequence of observations, $x_1, x_2, \dots, x_t, \dots$ for $t = 1,2,\dots$, and $x_t \in \mathcal{X}$.
Each $x_t$ is generated independently following some distribution $D_t \in \mathcal M_1^+(\X)$. If there exists $t^*$
such that for $i < t^*$ and $j \geq t^*$ we have $D_i \neq D_j$, then $t^*$ is a change point, and our task is to detect it; in practice, a $D_t$ typically generates a range of i.i.d.\ observations.
We note that these definitions place few assumptions on the type of data, that is, we only require the data to reside in a topological space.

\tb{Kernel mean embedding.} Let $\H$ be a reproducing kernel Hilbert space (RKHS) on $\X$, which means
that the linear evaluation functional $\delta_x : \H \to \R$ defined by
$\delta_x(f) = f(x)$ is bounded for all $x \in \X$ and $f\in\H$. By the Riesz
representation theorem \citep{reed1972functional}, there exists for each $x \in
	\mathcal{X}$ a unique vector $\phi(x) \in \mathcal H$ such that for every $f \in
	\mathcal H$ it holds that $f(x) = \delta_x(f) = \langle f, \phi(x) \rangle$. The
function $\phi(x)$ is the reproducing kernel for $x$ and also called feature
map; it has the canonical form $x \mapsto k(\cdot,x)$, with the function $k :
	\mathcal{X} \times \mathcal{X} \to \mathbb R$ the reproducing kernel associated
to~$\mathcal H$. With this kernel, it holds that $k(x_1,x_2) = \langle \phi(x_1), \phi(x_2) \rangle = \langle k(\cdot, x_1), k(\cdot, x_2) \rangle$ for all $x_1, x_2 \in \mathcal X$~\citep{steinwart08support}.
The mean embedding of a probability measure $\P \in \mathcal{M}_1^+(\X)$ is the element
$\mu(\P) \in \mathcal H$ such that $\mathbb E_{X\sim\P}\left[f(X)\right] =
	\langle f, \mu(\P)\rangle$ for all $f \in \mathcal H$. The mean embedding
$\mu(\P)$ exists if $k$ is measurable and
bounded~\citep[Prop.~2]{sriperumbudur10hilbert}, which we assume throughout the article.

\tb{Maximum mean discrepancy.}
MMD is defined by $\mathrm{MMD}(\P,\Q) = \|\mu(\P)-\mu(\Q)\|$, where $\mu(\P), \mu(\Q) \in \H$ are the mean embeddings of $\P,\Q \in \mathcal M_1^+(\X)$, respectively.

Let $X\sim\P$, $Y\sim\Q$ and $X'$, $Y'$ independent copies of $X$, $Y$, respectively.
The squared population MMD \citep[Lemma 6]{gretton2012kerneltwosample} then takes the form
\begin{align*}
	\MMD^2(\P, \Q) & = \mathbb{E}\left[k(X,X')\right] + \mathbb{E}\left[k(Y,Y')\right] - 2 \mathbb{E}\left[k(X,Y)\right],
\end{align*}
where the expectations are taken w.r.t.\ to all sources of randomness.
For observations $\hat \P_m = \{x_1,\dots,x_m\} \stackrel{\text{i.i.d.}}{\sim} \P$ and $\hat \Q_n = \{y_1,\dots,y_n\} \stackrel{\text{i.i.d.}}{\sim} \Q$,
a biased estimator is obtained by replacing the population means with their empirical counterparts
\begin{align}
	\label{eq:biased-mmd}
	\mathrm{MMD}^2\left(\hat \P_m,\hat \Q_n\right) & = \frac{1}{m^2}\sum_{i,j=1}^mk(x_i,x_j) +
	\frac{1}{n^2}\sum_{i,j=1}^nk(y_i,y_j)
	- \frac{2}{mn}\sum_{i,j=1}^{m,n}k(x_i,y_j).
\end{align}
The runtime complexity of (\ref{eq:biased-mmd}) is in $\O\left(m^2+n^2\right)$. We will base our proposed approximation on \eqref{eq:biased-mmd}.

\tb{Two-sample testing.} To decide whether the value of $\MMD\left(\hat \P_m,\hat \Q_n\right)$ indicates a significant difference between
$\P$ and $\Q$, one tests the null hypothesis $H_0
	: \P=\Q$ versus its alternative $H_1 : \P\neq \Q$ by defining an acceptance
region for a given level $\alpha \in (0,1)$, which takes the form $ \MMD\left(\hat \P_m,\hat \Q_n\right)
	< \epsilon_{\alpha}$. One rejects $H_0$ if the test statistic exceeds the
threshold. The level $\alpha$ is a bound for the probability that the tests
rejects $H_0$ incorrectly \citep{casella1990statisticalinference}. Assuming that $k$ is nonnegative and bounded by $K>0$, that is, $0\le k(x,y)\le K$ for all $x,y \in \X$,~\citet[Corollary~9]{gretton2012kerneltwosample} provides the distribution-free threshold
$\epsilon_\alpha$ for the case that both samples $\hat \P_m$ and $\hat \Q_m$ have the same size
($m=n$) as
\begin{align}
	\label{eq:mmd-original-threshold}
	\MMD\left(\hat \P_m,\hat \Q_m\right) < \sqrt{\frac{2K}{m}}\left(1+\sqrt{2\log \frac{1}{\alpha}}\right).
\end{align}
Computing \eqref{eq:mmd-original-threshold} costs $\O(1)$. As the change detection setting requires the case that $m\neq n$, we extend their threshold accordingly in what follows.

\section{Our proposed algorithm}
\label{sec:MMDEW}

We introduce MMDEW in three steps. We first extend the threshold for the MMD two-sample test, \eqref{eq:mmd-original-threshold}, to samples
of unequal sizes (Section~\ref{sec:threshold}).
We then introduce our data structure that enables the efficient computation of MMD on data streams (Section~\ref{sec:MMDEW-data-structure}). Last, we describe the complete algorithm in
Section~\ref{sec:MMDEW-alg}.

\subsection{Threshold for the hypothesis test}
\label{sec:threshold}
Given a sequence of observations $\{x_1,\dots,x_t\}$ up until time $t$ our goal is to test the null hypothesis $\P = \Q$ for any two neighboring windows $X\cdot Y =
	\{x_1,\dots,x_i\}\cdot\{x_{i+1},\dots,x_t\}$, with $i=1,\dots,t-1$.
Our following proposition extends \citet[Theorem~8]{gretton2012kerneltwosample}, which considers the case $m=n$, giving the
distribution-free acceptance region for $m\neq n$ (corresponding to the setting that one generally encounters in change detection). The proof is deferred to Appendix \ref{sec:proof-hyp-test}.

\begin{proposition} \label{prop:hyp-test}
	Let $\P, \Q \in \mathcal M_1^+(\mathcal X)$,
	$\hat \P_m = \{x_1,\ldots,x_m\}\stackrel{\text{i.i.d.}}{\sim} \P$,
	$\hat \Q_n = \{y_1,\ldots,y_n\}\stackrel{\text{i.i.d.}}{\sim} \Q$.
	Assume that $0\leq k(x,y) \leq K$ for all
	$x,y \in \X$ and $t > 0$. Then a hypothesis test of level at most $\alpha >0$
	for $\P=\Q$ has the acceptance region
	\begin{align}
		\label{eq:mmd-new-threshold}
		\MMD\left(\hat \P_m, \hat \Q_n\right) < \sqrt{\frac K m + \frac K n }\left(1+ \sqrt{2\log \alpha^{-1}}\right) =: \epsilon_\alpha.
	\end{align}
\end{proposition}

Note that, when considering multiple possible change points, one needs to account for multiple testing in order to achieve an overall level of size $\alpha$. For example, one may adjust $\epsilon_\alpha$ through Bonferroni correction ($\epsilon_{\alpha}' = {\epsilon_{\alpha}}/{(t-1)}$) by dividing by the total number of tests. %

To perform change detection with MMD, it is natural to consider the stopping time
\begin{align}
	T = \inf\left\{t : \max_{n=1,\ldots,t-1} \left[\MMD\left(\P_m,\Q_n\right) \ge \epsilon_\alpha\right] = 1\right\}, \label{eq:mmd-stopping-time}
\end{align}
for $m=t-n$, the empirical measures $\P_m = \{x_1,\ldots,x_m\}$, $\Q_n = \{x_{m+1},\ldots,x_t\}$, and the brackets equal to one if their argument is true and zero otherwise \citep{graham1994concretemathematics}; we note that $\epsilon_\alpha := \epsilon_\alpha(m,n)$ depends on the respective sizes of the subsamples considered. In other words, a change is indicated by the first time any MMD estimated across all splits exceeds its threshold. However, due to the quadratic runtime requirements of MMD, the computation of \eqref{eq:mmd-stopping-time} costs $\O\!\left(t^3\right)$ for each new observation.

We now introduce our novel data structure that allows considering multiple possible change points efficiently.

\subsection{Proposed data structure}
\label{sec:MMDEW-data-structure}

One common method to obtain a good runtime complexity in change detection algorithms is to slice the data into windows of exponentially increasing sizes \citep{DBLP:conf/sdm/BifetG07}. Recent observations are collected in smaller windows, and older observations are grouped into larger windows. This leads to a fine-grained change detection in the recent past and more coarse-grained change detection in the distant past.

Our new data structure adopts this concept and, at the same time, facilitates the computation of MMD. In what follows, we first describe the properties of the proposed data
structure. Then, we show how to update the data structure and explain its use for
change detection.

\subsubsection{Properties}
We use $2$ as the basis for the exponential slicing. Then, after observing $t$ elements, the number of windows
stored in the data structure corresponds to the number of ones in the binary
representation of $t$. We may thus index the windows as $B_l, \dots, B_0$ (in
decreasing order), with the largest position being $l = \lfloor \log_2 t \rfloor$. A window does not exist if the binary representation of $t$ at this position is
zero.

If it exists, a window $B_s = \left(\mathrm X_s, \mathrm{XX}_s, \mathrm{XY}_s\right)$ at position $s = 0,\dots, l$ stores $2^s$ observations
\begin{align}
	\mathrm{X}_s = \left\{x_1^{(s)},\dots,x_{2^s}^{(s)}\right\}, \label{eq:def-x-i}
\end{align}
together with the summaries
\begin{align}
	\mathrm{XX}_s & = \sum_{i,j=1}^{2^s}k\left(x_i^{(s)},x_j^{(s)}\right), \label{eq:xx-values}                                                                                                                                                                                                      \\
	\mathrm{XY}_s & = \Bigg\{\underbrace{\sum_{i=1}^{2^s}\sum_{j=1}^{2^{s+1}}k\left(x_i^{(s)},x_j^{(s+1)}\right)}_{=:   \mathrm{XY}_s^{s+1}},\dots,\underbrace{\sum_{i=1}^{2^s}\sum_{j=1}^{2^l}k\left(x_i^{(s)},x_j^{(l)}\right)}_{=: \mathrm{XY}_s^{l}}\Bigg\},\hspace{0.25cm} \label{eq:xy-values}
\end{align}
where $\mathrm{XX}_s \in \R$ is the sum of the kernel $k$ evaluated on all pairs of the
window's own observations, and $\mathrm{XY}_s$ stores a list
of sums of the
kernel evaluated on the window's own observations and the observations in windows
coming before it.\footnote{Note that the superscript $(s)$ of the $x_i^{(s)}$-s indicates the corresponding window $B_s$.} Storing a list enables the efficient merging of windows, elaborated in Lemma~\ref{lemma:merging}. The length of the list $\mathrm{XY}_s$ equals the number
of windows having observations older than window $B_s$ and is at most $\lfloor
	\log_2 t \rfloor$. We use $\mathrm{XY}_i^j$ to represent the entry in
$\mathrm{XY}_i$ that refers to the window $B_j$. Specifically, in
\eqref{eq:xy-values}, $\mathrm{XY}_s^{s+1}$ stores the interaction of $B_s$
with $B_{s+1}$; similarly, $\mathrm{XY}_s^{l}$ stores its interaction with~$B_{l}$.

\begin{remark}\label{remark:value-l}
	Given a stream of data $x_1,x_2,\ldots,x_t$, \eqref{eq:def-x-i} corresponds to the mapping $x_i^{(s)} = x_\ell$, with $\ell = \sum_{j=s+1}^{\lfloor\log t\rfloor}2^j[B_j \text{ exists}] +i$, where the bracket is one if the argument is true and zero otherwise (using Iverson's convention; \citealt{graham1994concretemathematics}). A bucket $B_j$ exists if the $j$-th right-most digit in the binary expansion of $t$ is $1$.
\end{remark}

We summarize two of the main properties of the data structure as lemmas.
Lemma~\ref{prop:mmd-neighboring-buckets} establishes that one can compute the
value of MMD between two windows with constant complexity. The proof follows from
comparing \eqref{eq:xx-values} and \eqref{eq:xy-values} with~\eqref{eq:biased-mmd}. Lemma~\ref{lemma:merging} shows that windows can be merged with logarithmic runtime
complexity. These results provide our first steps towards efficiently computing MMD in a data stream.

\begin{lemma}
	\label{prop:mmd-neighboring-buckets}
	Let $B_{s+1}$ and $B_s$ be any two neighboring windows with elements $\mathrm{X}_{s+1} = \left\{x_1^{(s+1)},\dots,x_{2^{s+1}}^{(s+1)}\right\}$ and $\mathrm{X}_{s} = \left\{x_1^{(s)},\dots,x_{2^s}^{(s)}\right\}$, and sums as defined by \eqref{eq:xx-values} and~\eqref{eq:xy-values}, respectively. Then
	\begin{align}
		\MMD^2\left(\mathrm X_{s+1},\mathrm X_s\right) & = \frac{1}{(2^{s+1})^2} \mathrm{XX}_{s+1} + \frac{1}{(2^{s})^2}\mathrm{XX}_s - \frac{2}{(2^{s+1})(2^{s})}\mathrm{XY}_s^{s+1}, \label{eq:mmd-neighboring-buckets}
	\end{align}
	with a computational complexity of $\O(1).$
\end{lemma}

\begin{lemma}
	\label{lemma:merging}
	Merging two windows $B_{s+1}$ and $B_s$ into a new window $B'$, such that $B'$ stores \eqref{eq:def-x-i}, \eqref{eq:xx-values}, and \eqref{eq:xy-values} costs $\O\left(\log t\right)$.
\end{lemma}

Besides showing the result, the proof of Lemma~\ref{lemma:merging} illustrates the steps that allow merging windows efficiently.
\begin{proof}
	For computing $\mathrm{XX}'$, we use the symmetry of $k$ to obtain
	\begin{align}
		\mathrm{XX}' & = \sum_{i,j=1}^{2^{s+1}} k\left(x_i^{(s+1)},x_j^{(s+1)}\right) + \sum_{i,j=1}^{2^{s}} k\left(x_i^{(s)},x_j^{(s)}\right)  + \sum_{i=1}^{2^{s+1}}\sum_{j=1}^{2^s} k\left(x_i^{(s+1)}, x_j^{(s)}\right) + \sum_{i=1}^{2^{s}}\sum_{j=1}^{2^{s+1}} k\left(x_i^{(s)}, x_j^{(s+1)}\right) \nonumber \\
		             & = \sum_{i,j=1}^{2^{s+1}} k\left(x_i^{(s+1)},x_j^{(s+1)}\right) + \sum_{i,j=1}^{2^{s}} k\left(x_i^{(s)},x_j^{(s)}\right) + 2 \sum_{i=1}^{2^{s+1}}\sum_{j=1}^{2^s} k\left(x_i^{(s+1)}, x_j^{(s)}\right) = \mathrm{XX}_{s+1} + \mathrm{XX}_s + 2 \mathrm{XY}_s^{s+1}, \label{eq:merge-xx}
	\end{align}
	which has a runtime complexity in $\O(1)$.

	To compute $\mathrm{XY}'$, we note that $B_{s+1}$ stores the list $\mathrm{XY}_{s+1}$ of kernel evaluations corresponding to all windows coming before it. The same holds for $B_s$, for which the list has one more element, $\mathrm{XY}_s^{s+1}$, which was used in \eqref{eq:merge-xx}. All the elements in $\mathrm{XY}_s$ and $\mathrm{XY}_{s+1}$ are sums and thus additive; it suffices to merge both lists by adding their values element-wise, omitting $\mathrm{XY}_s^{s+1}$, and storing the result in $\mathrm{XY}'$. As each list has at most $\log t$ elements, merging them is in $\O\left(\log t\right)$.
\end{proof}

Specifically, the scheme facilitates the merging of windows of equal size, enabling us to establish the exponential structure outlined in the next section.

\subsubsection{Insertion of observations}
\label{sec:insert}

The structure is set up recursively. For each new observation, we create a new
window $B_0$, with $\mathrm{XX}_0$ as defined by \eqref{eq:xx-values} and
$\mathrm{XY}_0$ computed w.r.t.\ the already existing windows. If two windows have the
same size, we merge them by Lemma~\ref{lemma:merging}, which costs $\O\left(\log
	t\right)$. This yields $\lfloor \log t \rfloor$ windows
of exponentially increasing~sizes.

We illustrate the scheme in the following Example~\ref{ex:example-1} and the corresponding Figure~\ref{fig:example-1-vis}.

\begin{example} \label{ex:example-1}
	To set up the structure, we start with the first observation $x_1$ and create the first window $B_0$, with $\mathrm{XX}_0$ as defined by \eqref{eq:xx-values} and $\mathrm{XY}_0=\emptyset$.  When observing $x_2$, we similarly create a new window $B_0'$, now also computing $\mathrm{XY}_{0'}^0 =\{\mathrm{XY}_{0'}^0\}$.
	As $B_0$ and $B_0'$ have the same size, we merge them into $B_1$, computing $\mathrm{XX}_1$ with \eqref{eq:merge-xx}. No previous window exists so that $\mathrm{XY}_1 = \emptyset$. We repeat this for all new observations, for example, for $x_3$, one creates (a new) $B_0$, computing $\mathrm{XX}_0$ and $\mathrm{XY}_0=\{\mathrm{XY}_0^1\}$, which results in two windows, $B_1$ and~$B_0$.
\end{example}

\begin{figure}
	\centering
	\includegraphics[width=.58\linewidth]{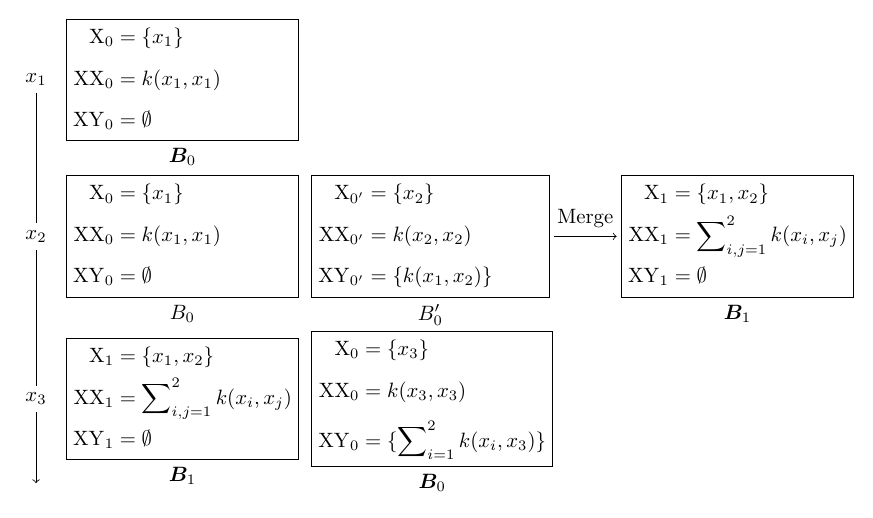}
	\caption{Schematic representation of Example~\ref{ex:example-1}. For a given step, the proposed scheme stores the windows in bold face.}
	\label{fig:example-1-vis}
\end{figure}

\subsubsection{MMD computation and change detection}
\label{sec:mmd-comp-change}
We now show that we can compute the MMD statistic \eqref{eq:biased-mmd} at positions between windows with a runtime complexity of $\O(\log t)$.

\begin{proposition}
	\label{prop:merging-possible}
	Let $B_l,\dots,B_{s+1},B_s,\dots,B_0$ be a given list of windows with corresponding elements $\mathrm{X}_i$, $i=0,\dots,l$, as defined in \eqref{eq:def-x-i}. For any split $s \in \{1,\ldots,l-1\}$, the computation of
	\begin{align}
		\MMD^2\left(\bigcup_{i=s+1}^l \mathrm{X}_i, \bigcup_{i=0}^{s} \mathrm{X}_i\right), \label{eq:mmd-union}
	\end{align}
	that is, the computation of MMD between the elements in windows coming before window $B_s$ and the elements in windows coming after (and including) $B_s$,
	has a runtime complexity of $\O\left(\log t\right)$ for  $0 < s < l$, with $s,l \in \mathbb N$.

	\begin{proof}
		To obtain \eqref{eq:mmd-union}, one recursively merges $B_s,\dots,B_{0}$ to $B_s'$ using Lemma~\ref{lemma:merging}, starting from the right, and similarly $B_{l},\dots,B_{s+1}$ to $B_l'$. One then obtains the statistic with Lemma~\ref{prop:mmd-neighboring-buckets}, and by setting    $\mathrm{XY}_{s'}^{l'} = \sum_{i=1}^{l-s} \mathrm{XY}_{s'}^i$, that is, by summing all elements in the $\mathrm{XY}_{s}'$-list of~$B_s'$.
		This concludes the proof as the logarithmic complexity was already established.
	\end{proof}
\end{proposition}

The application of the presented data structure for change detection is as follows. For each new observation, we estimate MMD at any position between windows and
compare it to the threshold $\epsilon_\alpha' =
	\frac {\epsilon_\alpha}{l}$ (with Bonferroni correction) from Proposition~\ref{prop:hyp-test}. We report a change when the value of MMD exceeds
the threshold. As there are at most $\log t$ windows, we have at most $\log t
	-1$ positions. Computing MMD for a position is in $\O\left(\log t\right)$ by
Proposition~\ref{prop:merging-possible}, and so the procedure has a total
runtime complexity of $\O\left(\log^2 t + t\right)$ per insert operation, where
the term linear in $t$ results from computing $\mathrm{XY}_0$ when inserting a new observation. We may equivalently consider the proposed method as providing a more coarse-grained estimate of \eqref{eq:mmd-stopping-time}, taking the form
\begin{align}
	T' = \inf\left\{t : \max_{s=1,\ldots,l} \left[\MMD\left(\bigcup_{i=s+1}^l \mathrm{X}_i, \bigcup_{i=0}^{s} \mathrm{X}_i\right) \ge \epsilon_\alpha\right] = 1\right\}, \label{eq:mmd-efficient-stopping-time}
\end{align}
with the $X_i$-s defined as in Proposition~\ref{prop:merging-possible}.\footnote{Recall that Remark~\ref{remark:value-l} gives the precise value of $l$. Further, some split positions $s$ may not exist.}

While the data structure in its current form allows to obtain the precise values of \eqref{eq:biased-mmd} in an incremental fashion, its runtime and memory complexity are $\O\left(t\right)$ for each new observation; these complexities are unsuitable for deploying the algorithm in the streaming setting. We reduce the runtime by subsampling within the windows, which we present together with the complete algorithm in the following section.

\subsection{MMDEW Algorithm}
\label{sec:MMDEW-alg}
Our algorithm builds upon the data structure discussed previously. But, we suggest that each window of size $2^s$, $s=0,\dots,l$, samples $s$ observations (of the total
$2^s$), that is, a logarithmic amount,  while keeping everything else as before.

In this section, we first analyze such subsampling and discuss its benefits. Afterwards, we present the complete algorithm. We refer to Appendix~\ref{sec:proof-count} for the proof of the following statement.

\begin{proposition}
	\label{prop:analysis-MMDEW}
	With subsampling, the number of terms in the sum  $\mathrm{XX}_l$ for a window at position $l$, $1 \le l$, $l \in \mathbb N$ is
	\begin{align}
		n_{\mathrm{XX}_l} & = 2^{l-1}\left(l^2-l+4\right) = \frac t2\left(\log_2^2t-\log_2t+4\right), \label{eq:num-xx}
	\end{align}
	with $t=2^l$ the number of observations of $B_l$. The number of terms of $\mathrm{XY}_l^l$ for windows of the same size, which occur prior to merging, is
	\begin{align}
		n_{\mathrm{XY}_l^l} = 2^{l} l =  t \log_2 t. \label{eq:num-xy}
	\end{align}
\end{proposition}

\begin{remark}
	The number of terms in the sums of \eqref{eq:biased-mmd} acts as a
	proxy for the quality of the estimate. It is optimal when no subsampling
	takes place; this number is $\O\left(t^2\right)$. When subsampling a logarithmic number of observations per window with our data structure (as we propose), one achieves polylogarithmic runtime and logarithmic memory complexity. At the same time, one achieves a better approximation quality than
	naively
	sampling a logarithmic number of observations without the
	summary data structure. While such sampling would also yield a memory complexity of $\O\left(\log t\right)$ when using the naive approach for change detection---that is, splitting the sample into two neighboring
	windows and computing $\MMD^2$---the number of terms in \eqref{eq:biased-mmd}
	would be $\O\left(\log^2 t\right)$. Proposition~\ref{prop:analysis-MMDEW} shows that the summary data structure improves upon this by a
	factor of approximately $t/2$ for $n_{\mathrm{XX}_l}$ and a factor of $t/\log_2
		t $ for  $n_{\mathrm{XY}_l^l}$ (we neglect logarithmic and constant terms in the former
	due to their small contribution).
\end{remark}

Algorithm~\ref{alg:optim-impl} now summarizes the complete algorithm, with $\MMD$ in
Line~\ref{alg:mmd-computation} referring to the computation of MMD as in
Proposition~\ref{prop:merging-possible}. MMDEW stores only a uniform sample of
size $l+1$, that is, of size logarithmic in the number of observations, while
keeping the respective $\mathrm{XX}_s$ and $\mathrm{XY}_s$, $s=0,\dots, l$,
computed before. With this approach, the number of samples in a window increases
by one each time the window is merged, and the memory complexity is logarithmic
in the number of observations. Note that one recovers the previous algorithm
(Section~\ref{sec:mmd-comp-change}) and therefore the precise value of
\eqref{eq:mmd-union} if one omits Line~\ref{alg:uniform-step}. Further, changes
in Line~\ref{alg:uniform-step} allow to adjust the subsampling, for example, the
user may defer the sampling until windows contain a minimum number of
observations, or choose a different function to control the sample size.

\setcounter{tmp}{\value{figure}}
\setcounter{figure}{0}
\renewcommand{\figurename}{Algorithm}
\begin{figure}
	\begin{algorithmic}[1]

		\Require{Data stream $x_1,x_2,\dots$, level $\alpha$}
		\Ensure{Change points in $x_1,x_2,\dots$; detection times}
		\State{$\mathit{windows} \gets \emptyset$} \Comment{List of windows}

		\For{each $x_i \in \{x_1,x_2,\dots\}$}
		\State{$X_0 \gets x_i$} \Comment{Initialize $B_0$}
		\State {$\mathrm{XX}_0  \gets k(x_i,x_i)$}  %
		\For{each $B_j \in \mathit{windows}$}
		\State{$\mathrm{XY}_0^j \gets \sum_{x_k^{(j)} \in B_j} k(x_i,x_k^{(j)})$}
		\EndFor
		\State{$B_0 = (\mathrm X_0,\mathrm{XX}_0,\mathrm{XY}_0)$}
		\State{$\mathit{windows} \gets windows \cup B_0$}
		\For{each split $s$ in $\mathit{windows} =  \{B_l,\dots,B_{s+1},B_s,\ldots, B_0\}$)} \Comment{Detect changes}
		\If{$\MMD\left(\bigcup_{j=s+1}^l\mathrm{X}_j, \bigcup_{j=0}^{s}\mathrm{X}_j\right) \geq \epsilon_\alpha'$}  \label{alg:mmd-computation}
		\State{\textbf{print} ``Change at $s$ detected at time $i$''}
		\State{$\mathit{windows} \gets B_s, \dots, B_0$}
		\Comment{Drop windows}
		\EndIf
		\EndFor
		\While{two windows have the same size $2^l$} \Comment{Maintain exponential structure}
		\State{Merge windows following Lemma~\ref{lemma:merging} into $B_{l+1}$}
		\State{Store a uniform sample of size $l+1$ in $\mathrm X_{l+1}$ of $B_{l+1}$}  \label{alg:uniform-step}
		\EndWhile
		\EndFor
	\end{algorithmic}
	\caption{Proposed MMDEW change detection algorithm.}\label{alg:optim-impl}
\end{figure}
\renewcommand{\figurename}{Figure}
\setcounter{figure}{\value{tmp}}

The following example illustrates the procedure.
Figure~\ref{fig:data-structure} expands upon Example~\ref{example:subsampling} and
shows the evolution of the data structure upon observing $x_1,\ldots,x_6$ and when merging windows.

\begin{example} \label{example:subsampling} We assume that there is a stream of
	i.i.d.\ observations $x_1, x_2, \dots$. Note that the i.i.d.\ assumption implies that
	there are no changes. MMDEW receives the first observation, $x_1$ and creates a
	window $B_0$ storing $x_1$, $\mathrm{XX}_0 = k(x_1,x_1)$, and $\mathrm{XY}_0 =
		\emptyset$. For the next observation, $x_2$, it creates a new window $B_{0'}$,
	storing $x_2$, $\mathrm{XX}_{0'} = k(x_2,x_2)$, and $\mathrm{XY}_{0'} =
		\{k(x_1,x_2)\}$ and detects no change. As $B_0$ and $B_{0'}$ have the same size,
	MMDEW merges them into window $B_1$, storing a sample of size $\log_2 2 = 1$,
	say, it stores $x_1$ and discards $x_2$, and computes $\mathrm{XX}_1 =
		k(x_1,x_1) + k(x_2,x_2) + 2 k(x_1,x_2)$, following~\eqref{eq:xx-values}. As no
	previous window exists, the computation of $\mathrm{XY}_1$ is not required. We
	see that the number of terms in $\mathrm{XX}_1$ equals four, while $B_1$ stores
	only one observation (established in Proposition~\ref{prop:analysis-MMDEW}).
	Next, the algorithm observes $x_3$ and creates a new window, $B_0$, storing
	$x_3$, $\mathrm{XX}_0=k(x_3,x_3)$, and computing $\mathrm{XY}_0$ to the window
	coming before, that is, $B_1$, so that $\mathrm{XY}_0=\{\mathrm{XY}_0^1\}$. In the
	next step, MMDEW receives $x_4$, again creating a new window $B_{0'}$. The
	algorithm now recursively merges the windows, that is, $B_0$ and $B_{0'}$ become
	$B_{1'}$, and $B_1$ and $B_{1'}$ then become~$B_2$. Upon receiving $x_5$, the algorithm creates a new window $B_0$, storing $x_5$, the kernel evaluation $k(x_5,x_5)$, and the interaction of $x_5$ with $\{x_1,x_4\}$ from $B_3$. We conclude the example with $x_6$, which leads to the creation of a new window $B_0'$. As in steps $2$ and $4$, $B_0$ and $B_0'$ will now be merged to obtain  $B_1$.
\end{example}

Algorithm~\ref{alg:optim-impl} has a runtime cost of $\O\left(\log^2 t\right)$ per insert operation and a total memory complexity of $\O\left(\log t\right)$.
This allows it to scale to very large data streams. Nevertheless, if one strictly requires constant time and memory, one can simply limit the number of windows at the expense of detecting changes only up to a certain time in the past. In the latter configuration, MMDEW fulfills the requirements for streaming algorithms laid out by \citet{domingos2003miningmassivedatastreams}.

\begin{figure}[t]
	\centering
	\includegraphics[width=.75\linewidth]{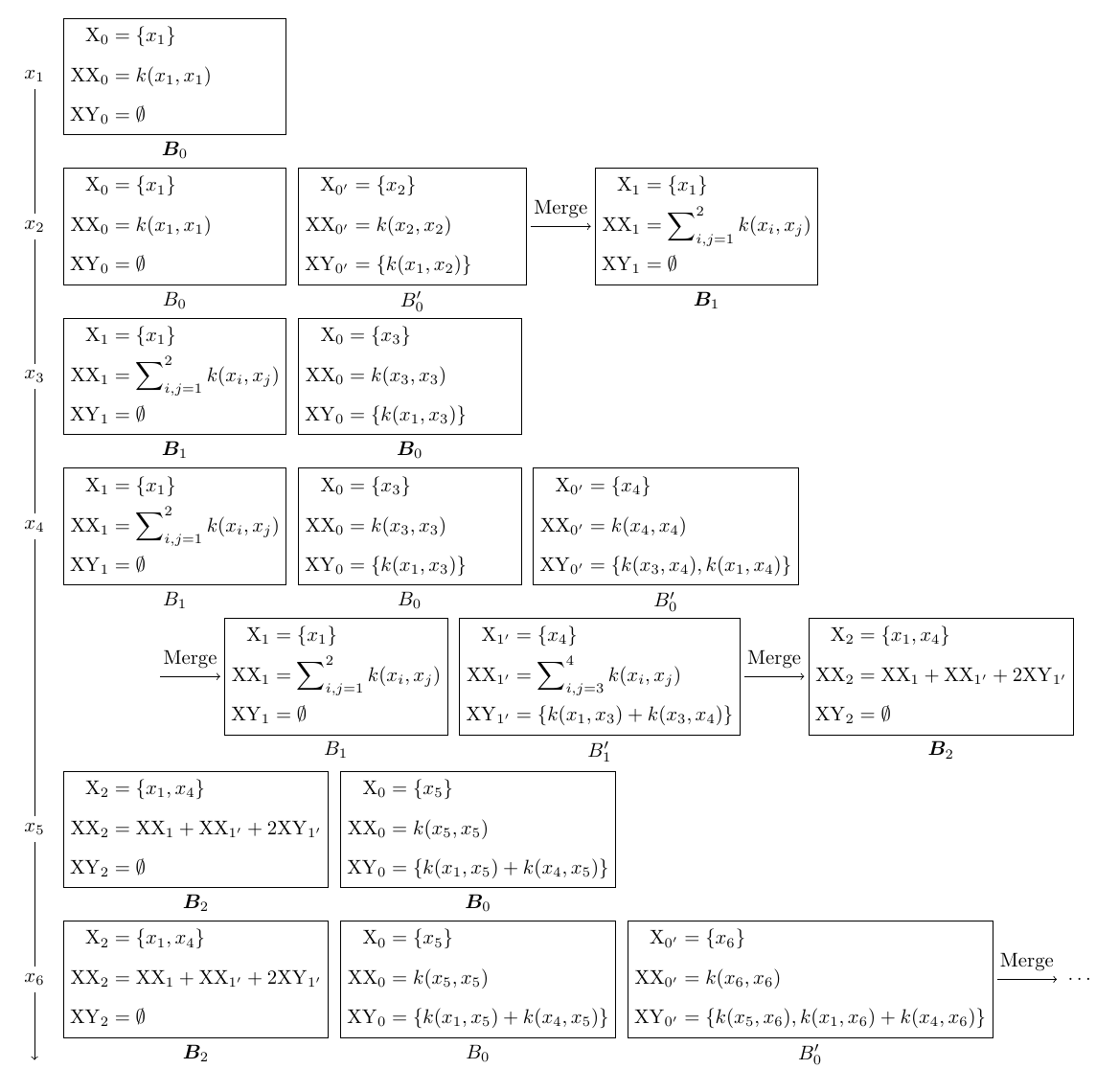}
	\caption{Set up of  data structure with subsampling upon inserting $x_1,\dots,x_6$. MMDEW stores the windows in bold face at the end of the merge operations. Observations $x_2$ and $x_3$ are not stored explicitly due to the sampling applied. $x_4$ is split into two lines for readability. See Example~\ref{example:subsampling} for a detailed discussion.}
	\label{fig:data-structure}
\end{figure}

\section{Experiments}
\label{sec:experiments}

This section showcases our approach on synthetic data (Section~\ref{sec:synthetic-data}) and on streams derived from real-world classification tasks (Section~\ref{sec:real-world-data}).
We ran all experiments on a server running Ubuntu 20.04 with 124GB RAM, and 32 cores with 2GHz each.

\subsection{Synthetic data}
\label{sec:synthetic-data}

To evaluate the average run length (ARL) and the mean time to detection (MTD) in a controlled environment, we first conduct experiments on synthetic data, comparing MMDEW to the MMD estimate \eqref{eq:biased-mmd} as baseline.\footnote{For computational reasons, we compute MMD as described in the discussion following Proposition~\ref{prop:mmd-neighboring-buckets}. F a fair comparison, we use the distribution-free bound of Proposition~\ref{prop:hyp-test} for both algorithms.} We also compare the runtime of MMDEW to that of existing change detectors. For an in-depth comparison of the ARL and MTD trade-off for kernel-based approaches with optimally chosen thresholds, see Appendix~\ref{appendix:edd-mtd-comparison}. For a comparison of MMDEW with univariate approaches, we refer to Appendix~\ref{appendix:univariate-comparison}.

\paragraph{ARL and MTD.}

The ARL quantifies the expected number of observations processed before a change detector flags a change, assuming $H_0$ holds. In the static setting, this corresponds to the type I error. Formally, for $\tilde T$ corresponding to the stopping time captured by Algorithm~\ref{alg:optim-impl}, that is, \eqref{eq:mmd-efficient-stopping-time} with subsampling applied, we are interested in $\E_{H_0} \tilde T$.

The error under the alternative ($H_1$ holds) is captured by the expected detection delay (EDD), also called ``mean time to detection (MTD)''. Specifically, a change detector processes a stream that contains a change at a known observation $\kappa$ ($\kappa = 1,\ldots,t$) and we want to know the delay until the change is reported $\E_{x_1,\ldots,x_{\kappa-1}\sim\P,x_\kappa,\ldots,x_t\sim\Q}\tilde T$, with $\P$ the pre-change distribution and $\Q \neq \P$ the post-change distribution. In other words, the EDD quantifies how many samples of $\Q$ must be processed to flag a change after having observed $\kappa-1$ samples from $\P$. Note that $\kappa$ must be chosen large enough to allow MMD to capture the difference in $\P$ and $\Q$, and we assume that the statistic does not exceed the treshold on the first $\kappa$ samples.
In a static setting, the EDD is comparable with the type II error. Note that existing literature \citep{xie12changepoint,wei22online} usually considers a fixed threshold $b$ for the stopping time for both ARL and EDD, while our threshold $\epsilon_\alpha$ depends on the position of the split considered. Experiments for a fixed value of $b$ are deferred to Appendix~\ref{appendix:edd-mtd-comparison}.

To approximate the ARL and the EDD in different scenarios, we simulate $5$-dimensional data distributed according to the multivariate normal
$\mathcal{N}\left(\bm 0, \mathbf I_{5}\right)$, the uniform
$\mathcal{U}[-\bm 1_{5},\bm1_{5}]$, the
$\mathrm{Laplace}\left(0,\sigma\mathbf I_{5}\right)$, and a mixed distribution, respectively.
The mixed distribution is taken to be $\mathcal{N}\left(\bm 0, \mathbf{I}_{5}\right)$ with
probability $0.3$ and $\mathcal{N}\left(\bm 0, \sigma^2\mathbf{I}_{5}\right)$
with probability $0.7$,
where $\bm 1_d$ denotes a vector of $d$ ones and
$\textbf{I}_d$ is the $d$-dimensional identity matrix. We
set $\sigma=3$.

To compute ARL, we consider $10{,}000$ observations distributed according to either the uniform, the Laplace, or the mixed distribution. Hence, the data does not contain any changes. For MTD, we first run both algorithms on $512$ ($=2^9$) and $1024$ ($=2^{10}$) observations, respectively, leading to MMDEW summarizing the data in one window in both cases. These observations are distributed according to $\mathcal{N}\left(\bm 0, \mathbf I_{5}\right)$ and then followed by either the uniform, Laplace, or mixed distribution. That is, we induce a change point at $\kappa = 513$ (resp.\ $\kappa = 1025$), and then count the number of observations processed from the new distribution until the algorithms report a change.

Figure~\ref{fig:arl-edd} collects the average results over $20$ repetitions. The left plot shows
that an increase in the level $\alpha \in (0,1)$ leads to a decrease in ARL. This is expected as the test becomes more sensitive, leading to more false positives. The baseline achieves a higher ARL but at the cost of an increased runtime.
The MTD plots (center and r.h.s.) mirror the ARL observation: The MTD decreases with increasing $\alpha$.
We further observe that the detection delay depends on the post-change distribution. The delay is comparably large when changing from the multivariate standard normal to the mixed distribution. This matches our intuition: the mixed distribution is relatively similar to the pre-change distribution, rendering it difficult to detect a change between them. For larger values of $\alpha$, that is, $\alpha \ge 0.2$, MMDEW performs similarly to the baseline in all cases. Comparing the MTD when the change happens after $512$ observations to the MTD when the change happens after $1024$ observations, the results show that more pre-change samples render the algorithms more sensitive to detecting changes, due to more samples improving the approximation of the mean embedding.
Overall, the results on these synthetic streams indicate that MMDEW is (i) robust to the choice of $\alpha$ and (ii) that $\alpha$ has the expected influence on the behavior of the algorithm.

\begin{figure}
	\centering
	\includegraphics[width=\mysize\textwidth ]{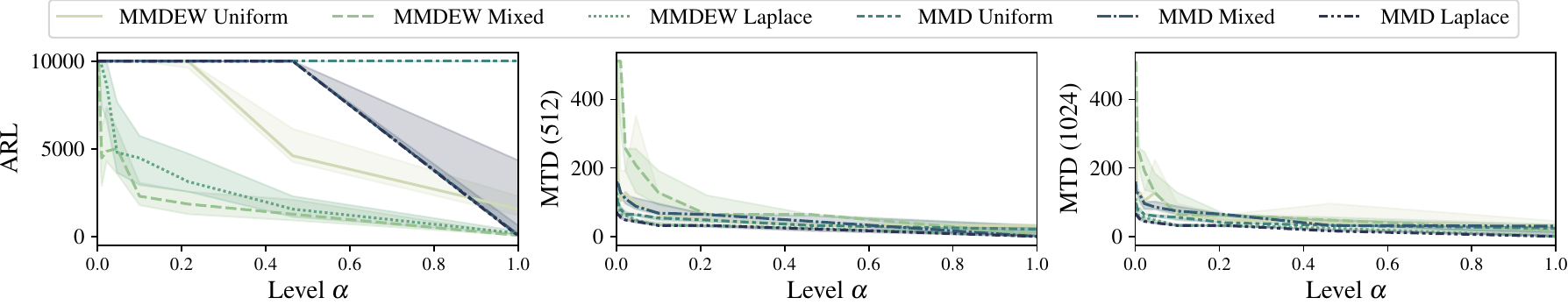}%
	\caption{Average run length (ARL) and expected detection delay / mean time to detection (EDD / MTD) of MMDEW on synthetically generated data.}
	\label{fig:arl-edd}
\end{figure}

\paragraph{Runtime.}

We now compare the runtime of MMDEW to that of its contenders and additionally validate the runtime guarantees that we derived analytically in Section~\ref{sec:MMDEW-alg}.

To this end, we generate a constant stream of
$10^6$ one-dimensional observations, that is, the observed stream contains no change.
Note that, while the dimensionality of the data affects the runtime depending on the used kernel, its influence is the same across all kernel-based algorithms, hence we limit our considerations to the univariate case.

Figure~\ref{fig:runtime} shows the average results over $10$ runs. The left
plot reveals that the fixed cost per insert of MMDEW is relatively large, as
processing a small number of observations requires comparably much time.
However, the runtime does not increase by much with the number of
observations. The figure also shows that the proposed algorithm's runtime is better than that of an alternate kernel-based method, Scan $B$-statistics,
where we use a window size of $\omega = 100$ in the runtime experiments. For $t>0.05\cdot 10^{6}$,
MMDEW also outperforms IBDD. Still, the other algorithms run faster than MMDEW but achieve a lower $F_1$ score in our later experiments.

The right plot of Figure~\ref{fig:runtime} verifies the analytically derived runtime of $\O \left(\log ^2 t\right)$ by fitting the corresponding curve ($t \mapsto c\log^2 t$) to the measured data with the least squares method.
The resulting mean squared error is approximately $10^{-6}$, which confirms the preceding asymptotic runtime analysis.

\begin{figure}
	\centering
	\includegraphics[width=\mysize\linewidth]{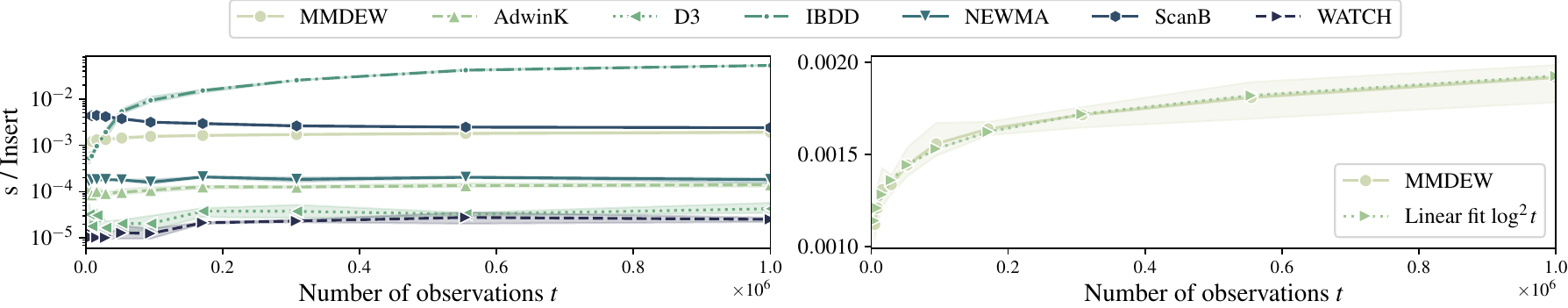}
	\caption{Comparison of runtimes per insert operation (l.h.s.) and least squares fit validating the theoretical runtime complexity of MMDEW w.r.t.\ the runtime observed in practice (r.h.s.).}
	\label{fig:runtime}
\end{figure}

\subsection{Real-world classification data}
\label{sec:real-world-data}

To obtain our change detection quality estimates, we use well-known classification data sets
and interpret them as streaming
data.\footnote{While MMDEW is not limited to Euclidean data, Euclidean data is the type of data most frequently encountered in practice, and our experiments target at this setting.} This is common in the literature, for example, \citet{DBLP:journals/inffus/FaithfullDK19,DBLP:conf/bigdataconf/FaberCSBJ21}, as only few high-dimensional annotated change detection data sets are
publicly available.

For each data set, we first order the observations by their classes; a change occurs if
the class changes. To introduce variation into the order of change points, we
randomly permute the order of the classes before each run but use the same
permutation across all algorithms. For preprocessing, we apply min-max scaling
to all data sets. Table~\ref{tab:dataset-overview} summarizes the data sets, where $n$ is the number of observations, $d$ is the data dimensionality, and \#CP is the number of change points.

\begin{table}
	\centering
	\caption{Overview of data sets.}
	\begin{tabular}{lrrr}
		\toprule
		Data set                                               & $n$        & $d$       & \#CPs \\
		\midrule
		CIFAR10 \citep{krizhevsky2009learning}                 & $60{,}000$ & $1{,}024$ & $9$   \\
		FashionMNIST \citep{DBLP:journals/corr/abs-1708-07747} & $70{,}000$ & $784$     & $9$   \\
		Gas \citep{vergara2012chemical}                        & $13{,}910$ & $128$     & $5$   \\
		HAR \citep{DBLP:conf/esann/AnguitaGOPR13}              & $10{,}299$ & $561$     & $5$   \\
		MNIST \citep{DBLP:journals/spm/Deng12}                 & $70{,}000$ & $784$     & $9$   \\
		\bottomrule
	\end{tabular}
	\label{tab:dataset-overview}
\end{table}

We run a grid parameter
optimization per data set and algorithm
and report the best result w.r.t.\ the $F_1$-score.
We note that such an optimization is difficult to perform in practice---here one typically prefers approaches with fewer or easy-to-set parameters---but allows a fair comparison.
Table~\ref{tab:hyperparameters} in Appendix~\ref{app:hyperparameters} lists all the parameters we tested. We note that the grid parameter optimization allowed us to obtain better $F_1$-scores than the heuristics proposed in~\citet{keriven2020mmdchangedetection} for NEWMA and Scan $B$-statistics.

We exclude the
squared time estimator of MMD due to its prohibitive runtime. For kernel-based
algorithms (MMDEW, NEWMA, and Scan $B$-statistics)
we use the Gaussian kernel $k(x,y) = \exp\left(-\gamma\|x-y\|^2\right)$ ($\gamma > 0$) and set $\gamma$ using the
median heuristic \citep{garreau18large} on the first 100
observations. The Gaussian kernel is universal \citep{steinwart08support,szabo2017characteristic} and allows, given
enough data, to detect any change in distribution as a universal kernel on a compact domain is characteristic~\citep[Theorem~5]{gretton2012kerneltwosample}.
We also supply the first 100
observations to competitors requiring data to estimate further parameters (IBDD, WATCH) upfront.

\paragraph{$F_1$-score, precision, and recall.}

We compute the precision, the recall, and the $F_1$-score, which are common to evaluate
change detection algorithms
\citep{li2019scanbstatistics,keriven2020mmdchangedetection,DBLP:journals/corr/abs-2003-06222,DBLP:conf/bigdataconf/FaberCSBJ21}.
Specifically, for a fixed $\Delta_T \in \mathbb N_{>0}$, we proceed as follows. If a change is detected, and there is an actual change point
within the $\Delta_T$ previous time steps, we consider it a true positive (tp).
If a change is detected, and there is no change point within the $\Delta_T$
previous steps, we consider it a false positive (fp). If no change is detected
within $\Delta_T$ steps of a change point, we consider it a false negative (fn).
We count at most one true positive for each actual change point. With these
definitions, the precision is $\mathrm{Prec} = \mathrm{tp} / (\mathrm{tp} +
	\mathrm{fp})$, the recall is $\mathrm{Rec} = \mathrm{tp} / (\mathrm{tp} +
	\mathrm{fn})$, and the $F_1$-score is their harmonic mean $  F_1 = 2 \cdot
	\left(\mathrm{Prec} \cdot \mathrm{Rec} \right)/\left( \mathrm{Prec} +
	\mathrm{Rec}\right)$.  Note that,
while some algorithms allow to infer where in the data a change happens,
including the proposed MMDEW, we only evaluate the time at which they report a change, as all tested approaches allow reporting this value.

\begin{figure*}
	\centering
	\includegraphics[width=\mysize\linewidth]{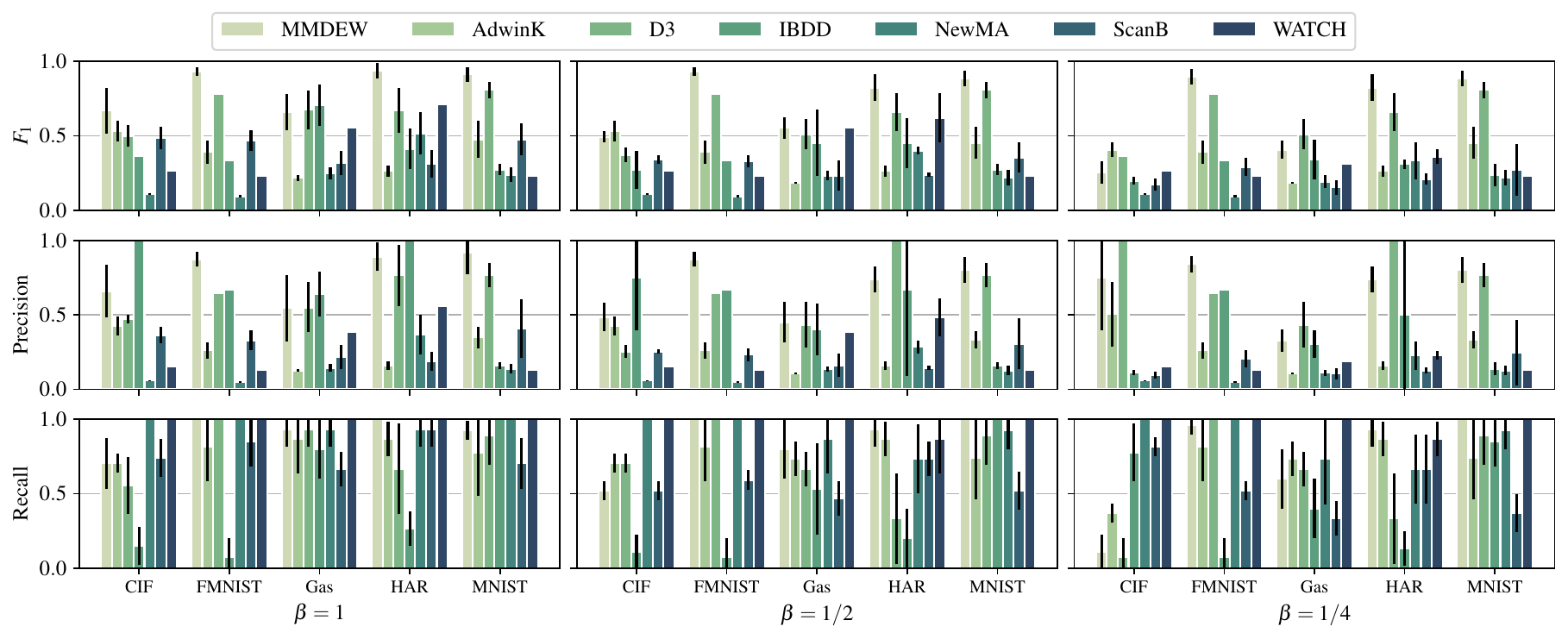}
	\caption{Average $F_1$-score, precision and recall. The bars show the standard deviation over $10$ permutations of the data.}
	\label{fig:results}
\end{figure*}

Figure~\ref{fig:results} shows our results. As $\Delta_T$ is an evaluation-specific parameter, we vary it relative
to the average distance between change points by a factor $\beta > 0$: Given a data
set of length $N$ with $n$ changes, we set $\Delta_T = \beta \cdot N / (n+1)$.
For $\beta = 1$ ($\Delta_T$ is equal to the average number of
steps between change points per respective data set),  MMDEW achieves a
higher $F_1$-score than all competitors on all data sets except for Gas, where it still obtains a competitive result. Throughout, the proposed algorithm obtains a good balance between precision and recall. Other approaches either have very low
precision (for example, less than 20\%), or an inferior recall and precision, down to a
few exceptions.
With a reduced $\beta$, that is, we allow only a shorter detection delay, the performance of all algorithms decreases on average. For $\beta = 1/2$, MMDEW achieves the best $F_1$ score also on four data sets, and, for $\beta = 1/4$ (the most challenging setting) on three of the tested data sets.

We conclude that the proposed method achieves very good results across all these experiments---especially when taking into account the fewer hyperparameters compared to the other approaches that we tested.

\paragraph{Percentage of changes detected and detection delay.}
To obtain a complete picture of the performance of MMDEW, we also report the ``percentage of changes detected'' (PCD), that is, the ratio of the number of reported changes and the number of actual change points, and its MTD on the data streams derived from real-world data. In our context, MTD coincides with the
expected detection delay.

\begin{figure}
	\centering
	\includegraphics[width=\mysize\linewidth]{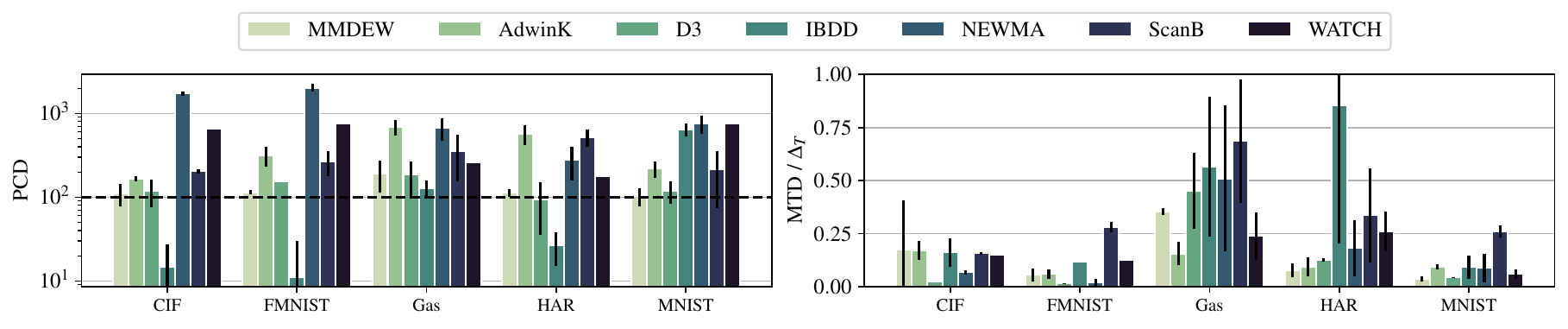}%
	\caption{Average of percentage of changes detected~(PCD) and of mean time to detection~(MTD). The dashed line indicates the optimum for PCD. For MTD lower values are better.}
	\label{fig:percent_changes_results}
\end{figure}

Figure~\ref{fig:percent_changes_results} collects our results. For PCD, results closer to 100\% are
better. Here, MMDEW is on par with the closest competitors and consistently,
that is, across all data sets, detects an approximately correct number of change
points. D3, NEWMA, Scan $B$-statistics, and WATCH detect too many change points in all
cases. This behavior is also reflected in their comparably large recall in Figure~\ref{fig:results}.

For MTD, lower values are better. Here, the classification-based D3 performs best in most of the cases. MMDEW performs a bit worse than D3 but better than
the other algorithms on most data sets, with the Gas data set the major exception.
As the experiments in Figure~\ref{fig:results} show, a
lower $\Delta_T$ tends to lead to a lower $F_1$-score of MMDEW. In other words,
MMDEW tends to detect changes with some delay, but it detects them consistently.

\section{Conclusions}
\label{sec:conclusions}
We introduced a novel change detection algorithm, MMDEW, that builds upon two-sample testing with MMD, which is known to yield powerful tests on many domains. %
To facilitate the efficient computation of MMD, we presented a new data structure, which allows to estimate MMD with polylogarithmic runtime and logarithmic memory complexity. Our experiments on standard benchmark data show that MMDEW obtains the best $F_1$-score on most data sets. At the same time, MMDEW only has two parameters---the level of the statistical test and the choice of kernel. This simplifies the proposed algorithm's application in real-world use cases.

\subsubsection*{Acknowledgments}
This work was supported by the German Research Foundation (DFG) Research Training Group GRK 2153: Energy Status Data---Informatics Methods for its Collection, Analysis and Exploitation and by the Baden-Württemberg Foundation via the Elite Program for Postdoctoral Researchers.

\bibliography{bib/curated,bib/collected}
\bibliographystyle{tmlr}

\appendix
\section{Proofs}

This section contains additional proofs. The proof of Proposition~\ref{prop:hyp-test} is in Section~\ref{sec:proof-hyp-test}. Proposition~\ref{prop:analysis-MMDEW} is proved in Section~\ref{sec:proof-count}.

\subsection{Proof of Proposition~\ref{prop:hyp-test}} \label{sec:proof-hyp-test}

Proposition~\ref{prop:hyp-test} follows from the more general result that we
state below. The statement and proof are similar to
\citet[Theorem~8]{gretton2012kerneltwosample} but do not assume $m=n$. Note that
we recover \citet[Theorem~8]{gretton2012kerneltwosample}  in the case that $m=n$.
We prove Proposition~\ref{prop:hyp-test} afterwards.

\begin{proposition}
	\label{prop:proposition_m_neq_n}
	Let $\P$, $\Q$, $\hat \P_m$, $\hat \Q_n$ be defined as in the main text, assume $0\leq k(x,y) \leq K$ for all $x,y \in \X$, $\P=\Q$, and $t > 0$. Then
	\begin{align*}
		P\left(\MMD\left(\hat\P_m,\hat \Q_n\right) - \left(\frac Km + \frac Kn \right)^\frac{1}{2} \geq t\right) \leq e^{-\frac{t^2mn}{2K(m+n)}}.
	\end{align*}

	\begin{proof}

		First, we bound the difference of $\MMD\left(\hat\P_m,\hat \Q_n\right)$ to its expected value. Changing a single one of either $x_i$ or $y_j$ in this function results in changes of at most $2\sqrt K / m$, and $2\sqrt K / n$, giving
		\begin{displaymath}
			\sum_{i=1}^{n+m}c_i^2 = 4K\frac{n+m}{nm}.
		\end{displaymath}
		We now apply the bounded differences inequality (recalled in Theorem~\ref{thm:bounded-diff}) to obtain
		\begin{align*}
			P\!\left(\MMD\left(\hat\P_m,\hat \Q_n\right)\hspace{-0.05cm}-\hspace{-0.05cm} \mathbb E\MMD\left(\hat\P_m,\hat \Q_n\right)\hspace{-0.05cm} \geq\hspace{-0.05cm} t\right) \hspace{-0.05cm}\leq\hspace{-0.05cm} e^{-\frac{t^2mn}{2K(m+n)}}.
		\end{align*}
		The last step is to bound the expectation, which yields
		\begin{align*}
			\MoveEqLeft\E\MMD\left(\hat\P_m,\hat \Q_n\right) = \mathbb E\Bigg(\frac{1}{m^2}\sum_{i,j=1}^mk(x_i,x_j) +
			\frac{1}{n^2}\sum_{i,j=1}^nk(y_i,y_j)-  \frac {1}{mn}\sum_{i,j=1}^{m,n}k(x_i,y_j) - \frac  {1}{mn}\sum_{j,i=1}^{n,m}k(y_j,x_i)\Bigg)^\frac{1}{2}                                    \\
			 & \leq \Bigg(\frac 1 m \mathbb Ek(X,X) + \frac 1 n \mathbb E k(Y,Y) + \frac 1 m (m-1)\mathbb Ek(X,Y) +\frac 1 n (n-1) \mathbb E k(Y,X)-2\mathbb E k(X,Y)\Bigg)^{\frac{1}{2}}       \\
			 & =\Bigg(\frac 1 m \mathbb Ek(X,X) + \frac 1 n \mathbb Ek(Y,Y)  - \frac 1 m \mathbb Ek(X,Y) -\frac 1 n \mathbb E k(X,Y)\Bigg)^{\frac{1}{2}}                                        \\
			 & =  \Bigg(\frac 1 m \mathbb E\left[k(X,X) - k(X,Y)\right]  + \frac 1 n \mathbb E\left[k(X,X)-k(X,Y)\right]\Bigg)^\frac{1}{2} \leq \left(\frac K m + \frac K n\right)^\frac{1}{2}.
		\end{align*}
		Inserting this into the previous inequality, we obtain the stated result.
	\end{proof}
\end{proposition}

Proposition~\ref{prop:hyp-test} is now a corollary of Proposition \ref{prop:proposition_m_neq_n}, which follows by setting $\alpha =
	e^{-\frac{t^2mn}{2K(m+n)}}$ and solving for $t$ to obtain a test of
level $\alpha$.

\subsection{Proof of Proposition~\ref{prop:analysis-MMDEW}}
\label{sec:proof-count}
To find $n_{\mathrm{XY}_l^l}$, we use our implementation of MMDEW and the On-Line Encyclopedia of Integer Sequences (OEIS)
to discover that $n_{\mathrm{XY}_l^l}$ follows the sequence $1,2,8,24,64,160,\dots$ for $l=0,1,2, \dots$. Thus
\begin{align}
	\label{eq:n-xy-ll}
	n_{\mathrm{XY}_l^l} = 2^{l} l,\quad \text{ for } l > 0
\end{align}
and $n_{\mathrm{XY}_0^0} = 1$~\citep{oeisA036289}.

To find $n_{\mathrm{XX}_l}$, notice that $n_{\mathrm{XX}_l}$ only changes when one merges two windows, which happens for windows of the same size $n_{\mathrm{XX}_{l-1}}$. The algorithm adds to this $2\cdot n_{\mathrm{XY}_{l-1}^{l-1}}$ terms, see \eqref{eq:merge-xx}, and, for $l = 0,1,2,\ldots$, we obtain the recurrence relation
\begin{align*}
	n_{\mathrm{XX}_l} = \begin{cases}
		                    1                                                                  & \quad\text{if $l = 0$,} \\
		                    4                                                                  & \quad\text{if $l = 1$,} \\
		                    2\cdot n_{\mathrm{XX}_{l-1}} + 2 \cdot n_{\mathrm{XY}_{l-1}^{l-1}} & \quad\text{if $l>1$},
	                    \end{cases}
\end{align*}
with $n_{\mathrm{XX}_{-1}} := 0$.
Now write
\begin{align}
	n_{\mathrm{XX}_l} = 2\cdot n_{\mathrm{XX}_{l-1}} + l \cdot 2^l - 2^l+2\cdot[l=0] + 2\cdot[l=1],\hspace{0.5cm} \label{eq:rec-relation}
\end{align}
where the brackets are equal to one if their argument is true and zero otherwise
(using Iverson's convention; \citealt{graham1994concretemathematics}). To find a
closed-form expression for \eqref{eq:rec-relation}, we define the ordinary generating function $A(z) =
	\sum_l a_lz^l$. Now, we multiply \eqref{eq:rec-relation} by $z_l$ and sum on
$l$, to obtain
\begin{align*}
	A(z) = \frac{-8z^3+2z-1}{(2z-1)^3}
\end{align*}
after some algebra, so that
\begin{align*}
	n_{\mathrm{XX}_l} & = [z^l]\frac{-8z^3+2z-1}{(2z-1)^3},
\end{align*}
where $[z^l]$ is the coefficient of $z^l$ in the series expansion of the generating
function $A(z)$. To extract coefficients, we first decompose $A(z)$ as
\begin{align*}
	A(z) = \frac{3}{1-2 z} - \frac{2}{(1-2z)^2} + \frac{1}{(1 - 2z)^3} - 1,
\end{align*}
which allows us to then find the coefficients as
\begin{align*}
	[z^l]\frac{3}{1- 2 z }                                       & \stackrel{(a)}{=} 3\cdot 2^l, &   &
	[z^l]-\frac{2}{(2 z - 1)^2} \stackrel{(b)}{=} -(l+1)2^{l+1}, &                               &
	[z^l]\frac{1}{(1-2z)^3} \stackrel{(c)}{=} (l+1)(l+2)2^{l-1},
\end{align*}
where \citet[Table~335]{graham1994concretemathematics} implies (a), (b) is \eqref{eq:n-xy-ll} shifted, and (c) is \citet{oeisA001788} shifted. We omit the last term as it corresponds to $[z^0]$, which we do not need. Now, adding all terms gives 
\begin{align}
	3 \cdot 2^l -(l+1)2^{l+1} + (l+1)(l+2)2^{l-1} = 2^{l-1}(l^2-l+4), 
\end{align}
concluding the proof.

\section{External results}

To proof Proposition~\ref{prop:hyp-test}, we recall McDiarmid's concentration inequality~\citep{vershynin2018}.

\begin{theorem}[Bounded differences inequality]
	\label{thm:bounded-diff}
	Let $X=\left(X_1,\dots,X_n\right)$ be a random vector with independent components. Let $f : \mathbb R^n \to \mathbb R$ be a measurable function. Assume that the value of $f(x)$ can change by at most $c_i > 0$ under an arbitrary change of a single coordinate of $x = (c_1,\dots,c_n) \in \mathbb R^n$. Then, for any $t > 0$, we have
	\begin{align*}
		P\{f(X) - \mathbb Ef(X) \geq t\} \leq \exp\left(-\frac{2t^2}{\sum_{i=1}^n c_i^2}\right).
	\end{align*}
\end{theorem}

\section{Hyperparameter optimization settings} \label{app:hyperparameters}

We collect the hyperparameter choices that we tested in our experiments on real-world classification data (Section~\ref{sec:real-world-data}) in Table~\ref{tab:hyperparameters} and
refer to the respective original publications for additional
information on the parameter settings.

\begin{table}[h]
	\centering
	\caption{Values chosen for the parameter optimization.}
	\label{tab:hyperparameters}
	\begin{tabular}{lcl}
		\toprule
		Algorithm & Parameters                      & Parameter values                                                                          \\
		\midrule
		MMDEW     & $\alpha$                        & $\alpha \in \{0.001,0.01,0.1, 0.2\}$                                                      \\
		ADWINK    & $\delta$, $k$                   & $\delta \in \{0.05, 0.1, 0.2, 0.9, 0.99\}$, $k\in\{0.01,0.02,0.05,0.1,0.2\}$              \\
		D3        & $\omega, \rho, \tau, \text{d}$  & $\omega \in \{100,200,500\}$,
		$\rho\in\{0.1,0.3,0.5\}$, $\tau\in\{0.7,0.8,0.9\}$, $\text{d}=1 $                                                                       \\
		IBDD      & $m, w$                          & $m \in \{10,20,50,100\}$, $w\in\{20, 100,200,300\}$                                       \\
		NEWMA     & $\omega,\alpha$                 & $\omega \in \{20,50,100\}$,  $\alpha \in \{0.01,0.02,0.05,0.1\}$                          \\
		Scan $B$  & $B, \omega,\alpha$              & $B\in\{2,3\}$, $\omega \in \{100,200,300\}$,  $\alpha \in \{0.01, 0.05\}$                 \\
		WATCH     & $\epsilon, \kappa, \mu, \omega$ & $\epsilon\in\{1, 2,3\}$, $\kappa \in \{25,50,100\}$, $\mu\in\{10,20,50,100, 1000,2000\}$, \\
		          &                                 & $\omega\in\{100, 250, 500, 1000\}$                                                        \\
		\bottomrule
	\end{tabular}

\end{table}

\section{Additional experiments}

In this section, we collect additional experiments on synthetic data. In Appendix~\ref{appendix:edd-mtd-comparison}, we align the ARLs of the kernel-based approaches and compare their respective EDD/MTD. In Appendix~\ref{appendix:univariate-comparison}, we show that our MMD-based approach detects changes in the covariance structure of multivariate data, which aggregated univariate approaches cannot detect reliably.

\subsection{EDD/MTD of kernel-based approaches on synthetic data}
\label{appendix:edd-mtd-comparison}

The following experiments compare the EDD of the kernel-based online change detection approaches considered in the main text and the related work for a fixed ARL on toy data, extending the experiments of \citet[Figure~4]{wei22online}. We note that all approaches considered in this section compute, for each new observation, a test statistic and compare the statistic to a threshold. If the statistic exceeds the threshold, a change is flagged. In the case of the proposed algorithm, multiple test statistics (one for each possible split) are computed. To allow for a comparison, we select the maximum MMD value across all splits, that is, for MMD, we consider the stopping rule
\begin{align}
	T'' = \inf\left\{t : \max_{s=1,\ldots,l} \MMD\left(\bigcup_{i=s+1}^l \mathrm{X}_i, \bigcup_{i=0}^{s} \mathrm{X}_i\right) \ge b\right\}, \label{eq:mmd-efficient-stopping-time-b}
\end{align}
with a fixed $b >0$, instead of \eqref{eq:mmd-efficient-stopping-time} as done in Section~\ref{sec:experiments}. Similarly, for MMDEW, we consider \eqref{eq:mmd-efficient-stopping-time-b} but with subsampling applied to the $X_i$-s.
The experimental setup is as follows.

To achieve a fixed target ARL $\E_{H_0}T$ for a given stopping time $T$, we run $25$ Monte Carlo simulations on $150{,}000$ samples from $\P = \mathcal N(\bm 0,\mathbf I_d)$ with $d=20$ and select $b$ as the $1-1/(\text{target ARL})$-quantile of the collected test statistics as threshold. For online kernel CUSUM, we set its parameters $B_{\text{max}}=50$ and $N=15$, matching the settings of \citet[Figure~4]{wei22online}. Similarly, for Scan $B$-statistics and NEWMA, we set $B_0=50$; the remaining parameters of NEWMA then follow from the heuristics detailed by the authors \citep{keriven2020mmdchangedetection}.

For approximating the EDD of MMDEW and NEWMA for a threshold $b$,
we draw $64$ and $400$ samples from $\P$, respectively, before sampling from $\Q$. Online kernel CUSUM and Scan $B$-statistics each receive $1{,}000$ samples from $\P$ upfront, for computing the variance estimate they require and to use as a reference sample. All approaches use the Gaussian kernel with the bandwidth set by the median heuristic \citep{garreau18large}.\footnote{Note that NEWMA uses random Fourier features \citet{DBLP:conf/nips/RahimiR07} to approximate the kernel.}

\begin{figure}
	\centering
	\includegraphics[width=\mysize\textwidth]{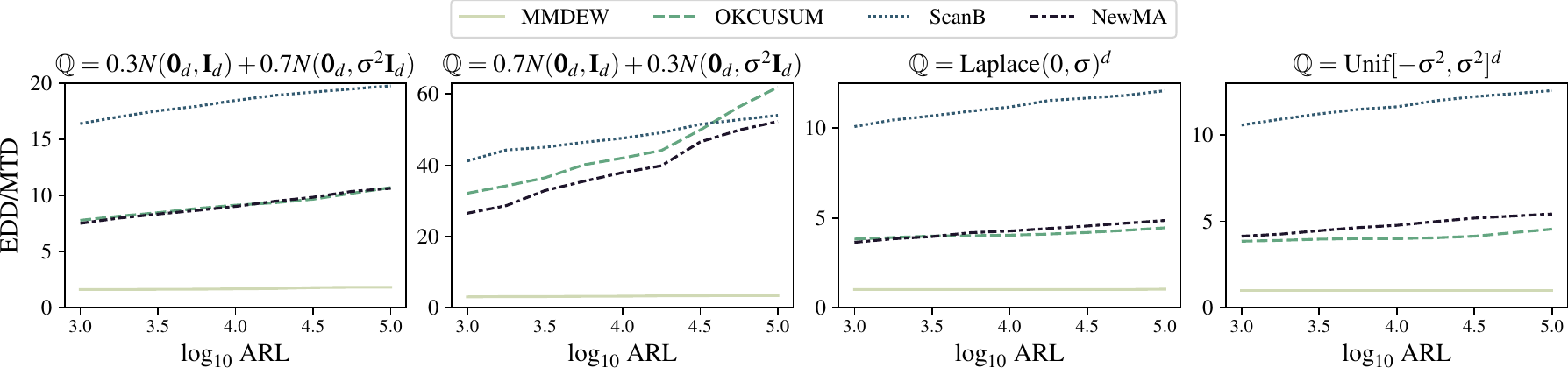}
	\caption{EDD/MTD of kernel-based change detectors with a pre-change distribution of $\P = \mathcal N(\bm 0,\mathbf I_d)$, $d=20$, and the indicated post-change distribution ($\sigma = 2$).}
	\label{fig:comp-arl-edd}
\end{figure}

Figure~\ref{fig:comp-arl-edd} collects our results, with each subfigure corresponding to a different post-change distribution, of which we sample and process $500$ elements to find the first time %
the test statistics exceeds the threshold. We consider the mean result over $100$ repetitions. The results show that OKCUSUM and NEWMA perform similarly across all experiments, with Scan $B$-statistics performing generally worse in three of the cases. MMDEW achieves the lowest EDD throughout. Specifically, on the mixed distribution $\Q = \gamma\mathcal N(0,\mathbf I_d) + (1-\gamma) \mathcal N(0,\sigma^2\mathbf I_d)$ ($\gamma=0.3$), the EDD of the proposed method is between $1.62$ and $1.82$. In the more challenging setting of $\gamma=0.7$, the EDD of MMDEW is between $3.06$ and $3.46$. Here, OKCUSUM performs second-best, with an EDD of $32.15$ for a target ARL of $1{,}000$ and $61.92$ for a target ARL of $100{,}000$. On the Laplace and Uniform distributions, the proposed method improves upon the results of OKCUSUM and NEWMA as well, albeit by a smaller margin.

While these experiments show that the proposed method improves upon the state-of-the-art, we note that the experiments require obtaining samples from $\P$, which is rarely feasible in practice. In this case, we recommend setting the threshold of MMDEW by the McDiarmid-based bound (Proposition~\ref{prop:hyp-test}), as we do in the experiments in the main text (Section~\ref{sec:experiments}).

\subsection{Comparison with univariate approaches}
\label{appendix:univariate-comparison}

In this section, we compare the test statistics of the proposed MMDEW, MMD, the Cramer-von-Mises change point model (CvM CPM; \citealt{ross12controlcharts,ross15cpmpackage}), and the recent non-parametric Focus \citep{romano23changedetection} approach on $20$-dimensional multivariate normal data with a mean and correlation shift, respectively. CvM CPM and Focus handle univariate data only. To run each on multivariate data, we run one instance per dimension and consider the means of their test statistics. For MMD and MMDEW, our settings are the same as detailed in Appendix~\ref{appendix:edd-mtd-comparison}. We set the pre-change distribution to $\P = \mathcal N(\bm 0_d,\mathbf I_d)$; the respective post-change distributions $\Q$ are indicated in Figure~\ref{fig:univariate-data}. For CPM, which updates all previous test statistics upon observing a new sample, we report the test statistics computed after processing $500$ samples from the pre-change distribution and $100$ samples from the respective post-change distribution; for all other approaches, we report the test statistic computed upon observing each sample.

\begin{figure}[h]
	\centering
	\includegraphics[width=\mysize\textwidth]{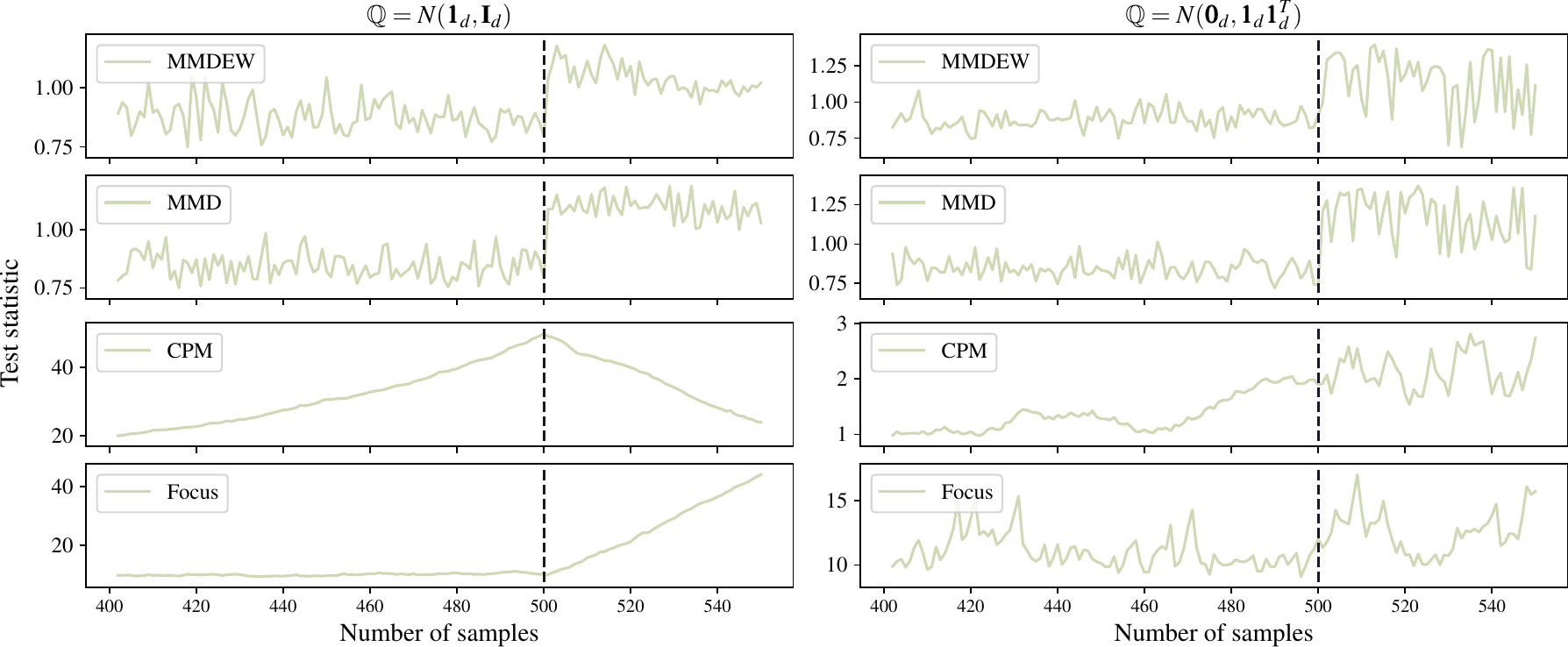}
	\caption{Maximum values of the respective test-statistics ($20$ repetitions, $d=20$). A change (indicated by a dashed line) occurs after $500$ samples, from $\mathcal N(\bm 0_d,\textbf I_d)$ to the distribution indicated on top of the columns, respectively. For the univariate approaches (CPM, Focus), we run one instance per dimension and consider the mean. }
	\label{fig:univariate-data}
\end{figure}

Our results are in Figure~\ref{fig:univariate-data}. A change in either the distribution mean (l.h.s.) or the correlation (r.h.s.) lead to an increase of the test statistic of MMDEW and MMD, respectively. Hence, these approaches allow detecting such changes. CPM and Focus correctly identify the change in mean, which is reflected in the univariate marginals they consider. CPM correctly identifies the change point, that is, the maximum value of the test statistic is at $500$.
When regarding the change in the correlation (the marginals the univariate approaches consider do not change), Focus' test statistic does not reflect the change point. Surprisingly, for CPM, the change in the correlation structure leads to a change in the test-statistic---but the change is identified incorrectly, with the maximum of the test statistic occurring after approximately $530$ samples. We conclude that using MMD or the proposed MMDEW is preferable to aggregating univariate change detectors when processing multivariate data, when the changes are not reflected in the marginals.

\vfill
\end{document}

%% file: preamble.tex
\usepackage{booktabs}
\usepackage{amsfonts,amsthm}
\usepackage{bm}
\usepackage{mathtools}
\usepackage[noend]{algpseudocode}
\mathtoolsset{showonlyrefs}
\usepackage{url}
\usepackage{enumitem}

\newcommand{\R}{\mathbb{R}}    %
\newcommand{\X}{\mathcal{X}}   %
\newcommand{\Q}{\mathbb{Q}}    %
\renewcommand{\P}{\mathbb{P}}  %
\renewcommand{\O}{\mathcal{O}} %
\newcommand{\E}{\mathbb{E}}    %
\renewcommand{\H}{\mathcal{H}} %
\newcommand{\tb}{\textbf}      %

\DeclareMathOperator{\MMD}{MMD}

\newtheorem{example}{Example}
\newtheorem{proposition}{Proposition}
\newtheorem{lemma}{Lemma}
\newtheorem{theorem}{Theorem}
\newtheorem{remark}{Remark}

\algrenewcommand\algorithmicrequire{\textbf{Input:}}
\algrenewcommand\algorithmicensure{\textbf{Output:}}

\setenumerate{labelindent=0em,leftmargin=1.6em,topsep=0cm,partopsep=0cm,parsep=0cm,itemsep=2mm}
\setitemize{labelindent=0em,leftmargin=1.3em,topsep=0cm,partopsep=0cm,parsep=0cm,itemsep=2mm}

%% file: main.bbl
\begin{thebibliography}{63}
\providecommand{\natexlab}[1]{#1}
\providecommand{\url}[1]{\texttt{#1}}
\expandafter\ifx\csname urlstyle\endcsname\relax
  \providecommand{\doi}[1]{doi: #1}\else
  \providecommand{\doi}{doi: \begingroup \urlstyle{rm}\Url}\fi

\bibitem[Abbasi \& Haq(2019)Abbasi and Haq]{abbasi19cusum}
Saba Abbasi and Abdul Haq.
\newblock Optimal {CUSUM} and adaptive {CUSUM} charts with auxiliary information for process mean.
\newblock \emph{Journal of Statistical Computation and Simulation}, 89\penalty0 (2):\penalty0 337--361, 2019.

\bibitem[Anguita et~al.(2013)Anguita, Ghio, Oneto, Parra, and Reyes{-}Ortiz]{DBLP:conf/esann/AnguitaGOPR13}
Davide Anguita, Alessandro Ghio, Luca Oneto, Xavier Parra, and Jorge~Luis Reyes{-}Ortiz.
\newblock A public domain dataset for human activity recognition using smartphones.
\newblock In \emph{{E}uropean {S}ymposium on {A}rtificial {N}eural {N}etworks ({ESANN})}, 2013.

\bibitem[Aronszajn(1950)]{aronszajn50theory}
Nachman Aronszajn.
\newblock Theory of reproducing kernels.
\newblock \emph{Transactions of the American Mathematical Society}, 68:\penalty0 337--404, 1950.

\bibitem[Berlinet \& Thomas-Agnan(2004)Berlinet and Thomas-Agnan]{berlinet04reproducing}
Alain Berlinet and Christine Thomas-Agnan.
\newblock \emph{Reproducing Kernel {H}ilbert Spaces in Probability and Statistics}.
\newblock Kluwer, 2004.

\bibitem[Bifet \& Gavald{\`{a}}(2007)Bifet and Gavald{\`{a}}]{DBLP:conf/sdm/BifetG07}
Albert Bifet and Ricard Gavald{\`{a}}.
\newblock Learning from time-changing data with adaptive windowing.
\newblock In \emph{{SIAM} {I}nternational {C}onference on {D}ata {M}ining ({SDM})}, pp.\  443--448, 2007.

\bibitem[Borgwardt et~al.(2020)Borgwardt, Ghisu, Llinares-L{\'o}pez, O'Bray, and Riec]{borgwardt20graph}
Karsten Borgwardt, Elisabetta Ghisu, Felipe Llinares-L{\'o}pez, Leslie O'Bray, and Bastian Riec.
\newblock Graph kernels: State-of-the-art and future challenges.
\newblock \emph{Foundations and Trends in Machine Learning}, 13\penalty0 (5-6):\penalty0 531--712, 2020.

\bibitem[Casella \& Berger(1990)Casella and Berger]{casella1990statisticalinference}
George Casella and Roger~L. Berger.
\newblock \emph{Statistical inference}.
\newblock Wadsworth \& Brooks/Cole, 1990.

\bibitem[Cheng \& Xie(2024)Cheng and Xie]{cheng24kernel}
Xiuyuan Cheng and Yao Xie.
\newblock Kernel two-sample tests for manifold data.
\newblock \emph{Bernoulli}, 30\penalty0 (4):\penalty0 2572--2597, 2024.

\bibitem[Cuturi(2011)]{cuturi11fast}
Marco Cuturi.
\newblock Fast global alignment kernels.
\newblock In \emph{{I}nternational {C}onference on {M}achine {L}earning ({ICML})}, pp.\  929--936, 2011.

\bibitem[Cuturi \& Vert(2005)Cuturi and Vert]{cuturi05contextfree}
Marco Cuturi and Jean-Philippe Vert.
\newblock The context-tree kernel for strings.
\newblock \emph{Neural Networks}, 18\penalty0 (8):\penalty0 1111--1123, 2005.

\bibitem[Dasu et~al.(2009)Dasu, Krishnan, Lin, Venkatasubramanian, and Yi]{DBLP:conf/ida/DasuKLVY09}
Tamraparni Dasu, Shankar Krishnan, Dongyu Lin, Suresh Venkatasubramanian, and Kevin Yi.
\newblock Change (detection) you can believe in: Finding distributional shifts in data streams.
\newblock In \emph{{I}nternational {S}ymposium on {I}ntelligent {D}ata {A}nalysis {(IDA)}}, volume 5772, pp.\  21--34, 2009.

\bibitem[de~Souza et~al.(2021)de~Souza, Parmezan, Chowdhury, and Mueen]{DBLP:journals/kais/SouzaPCM21}
Vin{\'{\i}}cius M.~A. de~Souza, Antonio Rafael~Sabino Parmezan, Farhan~Asif Chowdhury, and Abdullah Mueen.
\newblock Efficient unsupervised drift detector for fast and high-dimensional data streams.
\newblock \emph{Knowledge and Information Systems}, 63\penalty0 (6):\penalty0 1497--1527, 2021.

\bibitem[Deng(2012)]{DBLP:journals/spm/Deng12}
Li~Deng.
\newblock The {MNIST} database of handwritten digit images for machine learning research.
\newblock \emph{{IEEE} Signal Processing Magazine}, pp.\  141--142, 2012.

\bibitem[Domingos \& Hulten(2003)Domingos and Hulten]{domingos2003miningmassivedatastreams}
Pedro Domingos and Geoff Hulten.
\newblock A general framework for mining massive data streams.
\newblock \emph{Journal of Computational and Graphical Statistics}, 12\penalty0 (4):\penalty0 945--949, 2003.

\bibitem[Faber et~al.(2021)Faber, Corizzo, Sniezynski, Baron, and Japkowicz]{DBLP:conf/bigdataconf/FaberCSBJ21}
Kamil Faber, Roberto Corizzo, Bartlomiej Sniezynski, Michael Baron, and Nathalie Japkowicz.
\newblock {WATCH:} {W}asserstein change point detection for high-dimensional time series data.
\newblock In \emph{{IEEE} {I}nternational {C}onference on {B}ig {D}ata}, pp.\  4450--4459, 2021.

\bibitem[Faithfull et~al.(2019)Faithfull, Diez, and Kuncheva]{DBLP:journals/inffus/FaithfullDK19}
William~J. Faithfull, Juan Jos{\'{e}}~Rodr{\'{\i}}guez Diez, and Ludmila~I. Kuncheva.
\newblock Combining univariate approaches for ensemble change detection in multivariate data.
\newblock \emph{Information Fusion}, 45:\penalty0 202--214, 2019.

\bibitem[Fukumizu et~al.(2008)Fukumizu, Gretton, Sun, and Sch{\"o}lkopf]{fukumizu08kernel}
Kenji Fukumizu, Arthur Gretton, Xiaohai Sun, and Bernhard Sch{\"o}lkopf.
\newblock Kernel measures of conditional dependence.
\newblock In \emph{{A}dvances in {N}eural {I}nformation {P}rocessing {S}ystems ({NeurIPS})}, pp.\  498--496, 2008.

\bibitem[Gama(2010)]{gama2010knowledgediscovery}
Jo\~{a}o Gama.
\newblock \emph{Knowledge discovery from data streams}.
\newblock CRC Press, 2010.

\bibitem[Garreau et~al.(2018)Garreau, Jitkrittum, and Kanagawa]{garreau18large}
Damien Garreau, Wittawat Jitkrittum, and Motonobu Kanagawa.
\newblock Large sample analysis of the median heuristic.
\newblock Technical report, 2018.
\newblock \url{https://arxiv.org/abs/1707.07269}.

\bibitem[G{\"a}rtner et~al.(2003)G{\"a}rtner, Flach, and Wrobel]{gartner03graph}
Thomas G{\"a}rtner, Peter Flach, and Stefan Wrobel.
\newblock On graph kernels: Hardness results and efficient alternatives.
\newblock \emph{{C}omputational {L}earning {T}heory and {K}ernel {M}achines ({COLT})}, 2777:\penalty0 129--143, 2003.

\bibitem[G{\"{o}}z{\"{u}}a{\c{c}}ik et~al.(2019)G{\"{o}}z{\"{u}}a{\c{c}}ik, B{\"{u}}y{\"{u}}k{\c{c}}akir, Bonab, and Can]{DBLP:conf/cikm/GozuacikBBC19}
{\"{O}}mer G{\"{o}}z{\"{u}}a{\c{c}}ik, Alican B{\"{u}}y{\"{u}}k{\c{c}}akir, Hamed~R. Bonab, and Fazli Can.
\newblock Unsupervised concept drift detection with a discriminative classifier.
\newblock In \emph{{I}nternational {C}onference on {I}nformation and {K}nowledge {M}anagement ({CIKM})}, pp.\  2365--2368, 2019.

\bibitem[Graham et~al.(1994)Graham, Knuth, and Patashnik]{graham1994concretemathematics}
Ronald~L. Graham, Donald~E. Knuth, and Oren Patashnik.
\newblock \emph{Concrete Mathematics}.
\newblock Addison-Wesley, 1994.

\bibitem[Gretton et~al.(2012)Gretton, Borgwardt, Rasch, Sch\"{o}lkopf, and Smola]{gretton2012kerneltwosample}
Arthur Gretton, Karsten~M. Borgwardt, Malte~J. Rasch, Bernhard Sch\"{o}lkopf, and Alexander Smola.
\newblock A kernel two-sample test.
\newblock \emph{Journal of Machine Learning Research}, 13:\penalty0 723--773, 2012.

\bibitem[Harchaoui \& Capp{\'e}(2007)Harchaoui and Capp{\'e}]{harchaoui07retrospective}
Za{\"\i}d Harchaoui and Olivier Capp{\'e}.
\newblock Retrospective multiple change-point estimation with kernels.
\newblock In \emph{{IEEE/SP} {W}orkshop on {S}tatistical {S}ignal {P}rocessing}, pp.\  768--772, 2007.

\bibitem[Keriven et~al.(2020)Keriven, Garreau, and Poli]{keriven2020mmdchangedetection}
Nicolas Keriven, Damien Garreau, and Iacopo Poli.
\newblock N{EWMA}: {A} new method for scalable model-free online change-point detection.
\newblock \emph{IEEE Transactions on Signal Processing}, 68:\penalty0 3515--3528, 2020.

\bibitem[Kir{\'a}ly \& Oberhauser(2019)Kir{\'a}ly and Oberhauser]{kiraly19kernel}
Franz~J. Kir{\'a}ly and Harald Oberhauser.
\newblock Kernels for sequentially ordered data.
\newblock \emph{Journal of Machine Learning Research}, 20:\penalty0 1--45, 2019.

\bibitem[Krizhevsky et~al.(2009)Krizhevsky, Hinton, et~al.]{krizhevsky2009learning}
Alex Krizhevsky, Geoffrey Hinton, et~al.
\newblock Learning multiple layers of features from tiny images.
\newblock Technical report, 2009.
\newblock \url{https://www.cs.utoronto.ca/~kriz/learning-features-2009-TR.pdf}.

\bibitem[Li et~al.(2019)Li, Xie, Dai, and Song]{li2019scanbstatistics}
Shuang Li, Yao Xie, Hanjun Dai, and Le~Song.
\newblock Scan {$B$}-statistic for kernel change-point detection.
\newblock \emph{Sequential Analysis}, 38\penalty0 (4):\penalty0 503--544, 2019.

\bibitem[Lorden(1970)]{lorden70glr}
Gary Lorden.
\newblock On excess over the boundary.
\newblock \emph{Annals of Mathematical Statistics}, 41:\penalty0 520--527, 1970.

\bibitem[Lorden \& Pollak(2005)Lorden and Pollak]{lorden05change}
Gary Lorden and Moshe Pollak.
\newblock Nonanticipating estimation applied to sequential analysis and changepoint detection.
\newblock \emph{The Annals of Statistics}, 33\penalty0 (3):\penalty0 1422--1454, 2005.

\bibitem[M{\'e}rigot(2011)]{merigot11multiscale}
Quentin M{\'e}rigot.
\newblock A multiscale approach to optimal transport.
\newblock In \emph{Computer Graphics Forum}, volume~30, pp.\  1583--1592, 2011.

\bibitem[Muandet et~al.(2017)Muandet, Fukumizu, Sriperumbudur, and Sch{\"{o}}lkopf]{DBLP:journals/ftml/MuandetFSS17}
Krikamol Muandet, Kenji Fukumizu, Bharath~K. Sriperumbudur, and Bernhard Sch{\"{o}}lkopf.
\newblock Kernel mean embedding of distributions: {A} review and beyond.
\newblock \emph{Foundations and Trends in Machine Learning}, 10\penalty0 (1-2):\penalty0 1--141, 2017.

\bibitem[Page(1954)]{page54continuous}
E.~S. Page.
\newblock Continuous inspection schemes.
\newblock \emph{Biometrika}, 41:\penalty0 100--115, 1954.

\bibitem[Rahimi \& Recht(2007)Rahimi and Recht]{DBLP:conf/nips/RahimiR07}
Ali Rahimi and Benjamin Recht.
\newblock Random features for large-scale kernel machines.
\newblock In \emph{Advances in {N}eural {I}nformation {P}rocessing {S}ystems ({NeurIPS})}, pp.\  1177--1184, 2007.

\bibitem[Ramdas et~al.(2015)Ramdas, Reddi, P{\'{o}}czos, Singh, and Wasserman]{ramdas2015powerhypothesistest}
Aaditya Ramdas, Sashank~Jakkam Reddi, Barnab{\'{a}}s P{\'{o}}czos, Aarti Singh, and Larry~A. Wasserman.
\newblock On the decreasing power of kernel and distance based nonparametric hypothesis tests in high dimensions.
\newblock In \emph{Conference on {A}rtificial {I}ntelligence ({AAAI})}, pp.\  3571--3577, 2015.

\bibitem[Reed \& Simon(1972)Reed and Simon]{reed1972functional}
Michael Reed and Barry Simon.
\newblock \emph{Methods of modern mathematical physics. {I}. {F}unctional analysis}.
\newblock Academic Press, 1972.

\bibitem[Romano et~al.(2023)Romano, Eckley, and Fearnhead]{romano23changedetection}
Gaetano Romano, Idris~A Eckley, and Paul Fearnhead.
\newblock A log-linear non-parametric online changepoint detection algorithm based on functional pruning.
\newblock \emph{IEEE Transactions on Signal Processing}, 2023.

\bibitem[Ross(2015)]{ross15cpmpackage}
Gordon~J Ross.
\newblock Parametric and nonparametric sequential change detection in {R}: The cpm package.
\newblock \emph{Journal of Statistical Software}, 66:\penalty0 1--20, 2015.

\bibitem[Ross \& Adams(2012)Ross and Adams]{ross12controlcharts}
Gordon~J Ross and Niall~M Adams.
\newblock Two nonparametric control charts for detecting arbitrary distribution changes.
\newblock \emph{Journal of Quality Technology}, 44\penalty0 (2):\penalty0 102--116, 2012.

\bibitem[Scott(1991)]{scott1991feasibility}
David~W Scott.
\newblock Feasibility of multivariate density estimates.
\newblock \emph{Biometrika}, 78\penalty0 (1):\penalty0 197--205, 1991.

\bibitem[Sejdinovic et~al.(2013)Sejdinovic, Sriperumbudur, Gretton, and Fukumizu]{sejdinovic13equivalence}
Dino Sejdinovic, Bharath Sriperumbudur, Arthur Gretton, and Kenji Fukumizu.
\newblock Equivalence of distance-based and {RKHS}-based statistics in hypothesis testing.
\newblock \emph{Annals of Statistics}, 41:\penalty0 2263--2291, 2013.

\bibitem[Shewhart(1925)]{shewhart25application}
Walter~A. Shewhart.
\newblock The application of statistics as an aid in maintaining quality of a manufactured product.
\newblock \emph{Journal of the American Statistical Association}, 20\penalty0 (152):\penalty0 546--548, 1925.

\bibitem[Siegmund \& Venkatraman(1995)Siegmund and Venkatraman]{siegmund95changedetection}
D.~Siegmund and E.~S. Venkatraman.
\newblock Using the generalized likelihood ratio statistic for sequential detection of a change-point.
\newblock \emph{The Annals of Statistics}, 23\penalty0 (1):\penalty0 255--271, 1995.

\bibitem[Sloane(1999{\natexlab{a}})]{oeisA001788}
Neil James~Alexander Sloane.
\newblock Entry {A001788} in {The} {On-Line} {Encyclopedia} of {Integer} {Sequences}, 1999{\natexlab{a}}.
\newblock \url{https://oeis.org/A001788}.

\bibitem[Sloane(1999{\natexlab{b}})]{oeisA036289}
Neil James~Alexander Sloane.
\newblock Entry {A036289} in {The} {On-Line} {Encyclopedia} of {Integer} {Sequences}, 1999{\natexlab{b}}.
\newblock \url{https://oeis.org/A036289}.

\bibitem[Smola et~al.(2007)Smola, Gretton, Song, and Sch{\"o}lkopf]{smola07hilbert}
Alexander Smola, Arthur Gretton, Le~Song, and Bernhard Sch{\"o}lkopf.
\newblock A {H}ilbert space embedding for distributions.
\newblock In \emph{Algorithmic {L}earning {T}heory ({ALT})}, volume 4754, pp.\  13--31, 2007.

\bibitem[Sparks(2000)]{sparks00cusum}
Ross~S. Sparks.
\newblock {CUSUM} charts for signalling varying location shifts.
\newblock \emph{Journal of Quality Technology}, 32\penalty0 (2):\penalty0 157--171, 2000.

\bibitem[Sriperumbudur et~al.(2010)Sriperumbudur, Gretton, Fukumizu, Sch{\"o}lkopf, and Lanckriet]{sriperumbudur10hilbert}
Bharath Sriperumbudur, Arthur Gretton, Kenji Fukumizu, Bernhard Sch{\"o}lkopf, and Gert Lanckriet.
\newblock Hilbert space embeddings and metrics on probability measures.
\newblock \emph{Journal of Machine Learning Research}, 11:\penalty0 1517--1561, 2010.

\bibitem[Sriperumbudur \& Szab{\'o}(2015)Sriperumbudur and Szab{\'o}]{sriperumbudurszabo15optimal}
Bharath~K. Sriperumbudur and Zolt{\'a}n Szab{\'o}.
\newblock Optimal rates for random {F}ourier features.
\newblock In \emph{Advances in Neural Information Processing Systems (NeurIPS)}, pp.\  1144--1152, 2015.

\bibitem[Steinwart \& Christmann(2008)Steinwart and Christmann]{steinwart08support}
Ingo Steinwart and Andreas Christmann.
\newblock \emph{Support Vector Machines}.
\newblock Springer, 2008.

\bibitem[Szab\'{o} \& Sriperumbudur(2017)Szab\'{o} and Sriperumbudur]{szabo2017characteristic}
Zolt\'{a}n Szab\'{o} and Bharath~K. Sriperumbudur.
\newblock Characteristic and universal tensor product kernels.
\newblock \emph{Journal of Machine Learning Research}, pp.\  233:1--233:29, 2017.

\bibitem[Sz{\'e}kely \& Rizzo(2004)Sz{\'e}kely and Rizzo]{szekely04testing}
G{\'a}bor Sz{\'e}kely and Maria Rizzo.
\newblock Testing for equal distributions in high dimension.
\newblock \emph{InterStat}, 5:\penalty0 1249--1272, 2004.

\bibitem[Sz{\'e}kely \& Rizzo(2005)Sz{\'e}kely and Rizzo]{szekely05new}
G{\'a}bor Sz{\'e}kely and Maria Rizzo.
\newblock A new test for multivariate normality.
\newblock \emph{Journal of Multivariate Analysis}, 93:\penalty0 58--80, 2005.

\bibitem[van~den Burg \& Williams(2020)van~den Burg and Williams]{DBLP:journals/corr/abs-2003-06222}
Gerrit J.~J. van~den Burg and Christopher K.~I. Williams.
\newblock An evaluation of change point detection algorithms.
\newblock Technical report, 2020.
\newblock \url{https://arxiv.org/abs/2003.06222}.

\bibitem[Vergara et~al.(2012)Vergara, Vembu, Ayhan, Ryan, Homer, and Huerta]{vergara2012chemical}
Alexander Vergara, Shankar Vembu, Tuba Ayhan, Margaret~A Ryan, Margie~L Homer, and Ram{\'o}n Huerta.
\newblock Chemical gas sensor drift compensation using classifier ensembles.
\newblock \emph{Sensors and Actuators B: Chemical}, 166:\penalty0 320--329, 2012.

\bibitem[Vershynin(2018)]{vershynin2018}
Roman Vershynin.
\newblock \emph{High-dimensional probability}.
\newblock Cambridge University Press, 2018.

\bibitem[Wang \& Xie(2024)Wang and Xie]{wang24sequential}
Haoyun Wang and Yao Xie.
\newblock Sequential change-point detection: Computation versus statistical performance.
\newblock \emph{Wiley Interdisciplinary Reviews: Computational Statistics}, 16\penalty0 (1):\penalty0 e1628, 2024.

\bibitem[Watkins(1999)]{watkins99dynamic}
Chris Watkins.
\newblock Dynamic alignment kernels.
\newblock In \emph{Advances in {N}eural {I}nformation {P}rocessing {S}ystems ({NeurIPS})}, pp.\  39--50, 1999.

\bibitem[Wei \& Xie(2022)Wei and Xie]{wei22online}
Song Wei and Yao Xie.
\newblock Online kernel {CUSUM} for change-point detection.
\newblock Technical report, 2022.
\newblock \url{https://arxiv.org/abs/2211.15070}.

\bibitem[Xiao et~al.(2017)Xiao, Rasul, and Vollgraf]{DBLP:journals/corr/abs-1708-07747}
Han Xiao, Kashif Rasul, and Roland Vollgraf.
\newblock {Fashion-MNIST}: a novel image dataset for benchmarking machine learning algorithms.
\newblock Technical report, 2017.
\newblock \url{https://arxiv.org/abs/1708.07747}.

\bibitem[Xie et~al.(2023)Xie, Moustakides, and Xie]{xie23window}
Liyan Xie, George~V Moustakides, and Yao Xie.
\newblock Window-limited {CUSUM} for sequential change detection.
\newblock \emph{IEEE Transactions on Information Theory}, 69\penalty0 (9):\penalty0 5990--6005, 2023.

\bibitem[Xie \& Siegmund(2013)Xie and Siegmund]{xie12changepoint}
Yao Xie and David Siegmund.
\newblock Sequential multi-sensor change-point detection.
\newblock \emph{The Annals of Statistics}, 41\penalty0 (2):\penalty0 670--692, 2013.

\bibitem[Zaremba et~al.(2013)Zaremba, Gretton, and Blaschko]{DBLP:conf/nips/ZarembaGB13}
Wojciech Zaremba, Arthur Gretton, and Matthew~B. Blaschko.
\newblock B-test: {A} non-parametric, low variance kernel two-sample test.
\newblock In \emph{Advances in {N}eural {I}nformation {P}rocessing {S}ystems ({NeurIPS})}, pp.\  755--763, 2013.

\end{thebibliography}
